\documentclass[numbers,compress]{article}

\usepackage[preprint]{neurips_2026}

\usepackage[utf8]{inputenc}
\usepackage[T1]{fontenc}

\usepackage{url}
\usepackage{booktabs}
\usepackage{amsfonts}
\usepackage{nicefrac}
\usepackage{microtype}
\usepackage{xcolor}
\usepackage{amsmath,amssymb,amsthm,mathtools,bm}
\usepackage{enumitem}
\usepackage{graphicx}
\usepackage{float}
\usepackage{subcaption}
\definecolor{fan-red}{HTML}{D62727}

\definecolor{xinyue-blue}{HTML}{1F77B4}

\usepackage{xurl}
\definecolor{bluegray}{rgb}{0.4, 0.6, 0.8}
\definecolor{lightcarminepink}{rgb}{0.9, 0.4, 0.38}
\usepackage[pagebackref,breaklinks,colorlinks,citecolor=brown,linkcolor=lightcarminepink]{hyperref}

\usepackage{cleveref}
\crefname{assumption}{Assumption}{Assumptions}
\Crefname{assumption}{Assumption}{Assumptions}
\crefname{lemma}{Lemma}{Lemmas}
\Crefname{lemma}{Lemma}{Lemmas}
\crefname{corollary}{Corollary}{Corollaries}
\Crefname{corollary}{Corollary}{Corollaries}
\theoremstyle{plain}
\newtheorem{theorem}{Theorem}
\newtheorem{lemma}{Lemma}
\theoremstyle{definition}
\newtheorem{assumption}{Assumption}

\theoremstyle{remark}

\theoremstyle{plain}
\newcommand{\indep}{\perp\!\!\!\perp}
\newcommand{\zgt}{\bm{u}}
\newcommand{\zest}{\bm{z}}
\newcommand{\atil}{\tilde{\bm{a}}}
\renewcommand{\aa}{\bm{a}}
\renewcommand{\ss}{\bm{s}}
\newcommand{\xx}{\bm{x}}
\newcommand{\vv}{\bm{m}}     
\newcommand{\hh}{h}
\newcommand{\ff}{f}
\newcommand{\ww}{\bm{w}}
\newcommand{\FD}{\mathcal{F}}
\newcommand{\EE}{\mathcal{E}}
\renewcommand{\SS}{\mathcal{S}}
\newcommand{\Render}{R}       
\newcommand{\IDM}{I_\phi}     
\newcommand{\E}{\mathbb{E}}
\newcommand{\R}{\mathbb{R}}
\newcommand{\KL}{D_{\mathrm{KL}}}
\definecolor{scarred}{RGB}{180,75,75}
\definecolor{scarblue}{RGB}{60,110,180}
\definecolor{scargreen}{RGB}{65,140,95}

\title{SCAR: Self-Supervised Continuous Action Representation Learning}

\author{
Hongjia Liu$^{1, 2}$, Fan Feng$^{1}$, Minghao Fu$^{1}$, Xinyue Wang$^{1}$, Haofei Lu$^{2}$, Biwei Huang$^{1}$ \\
$^1$University of California, San Diego \\
$^2$KTH Royal Institute of
Technology
}

\begin{document}

\maketitle

\begin{abstract}
Despite the central role of action in embodied intelligence, learning transferable action representations from visual transitions remains a fundamental challenge, particularly when world models must generalize across embodiments under limited data. We argue that action is not merely an auxiliary conditioning signal, but a distinct representational factor that decouples the controllable change from embodiment-specific actuation. In this work, we propose SCAR, a joint inverse–forward dynamics framework for learning unified action representations across embodiments from visual transitions. Built on a pretrained generative backbone, SCAR uses an inverse dynamics model (IDM) to infer latent actions from latent observation pairs and a forward dynamics model (FDM) to predict future dynamics conditioned on them. To make the latent space transferable rather than a generic visual bottleneck, we regularize the latent action posterior toward a standard Gaussian prior to limit arbitrary visual encoding, and introduce adversarial invariance to suppress embodiment- and environment-specific nuisance factors. Experiments on the Procgen and Robotwin dataset show that the learned unified latent action representation serves as a stronger conditioning interface for world modeling than embodiment-specific raw actions, yielding improved cross-embodiment low-data adaptation and cross-task transfer. Taken together, these results suggest that action can be learned as a shared representation of controllable change across embodiments, providing an interface for more transferable and generalizable world models.
\end{abstract}

\section{Introduction}
Embodied agents must represent not only what is present in the scene but also how the body can act to change it. In cognitive science and neuroscience, this action-oriented aspect is discussed through the notion of \emph{body schema}: a sensorimotor representation that supports action planning and control, distinct from purely perceptual or semantic descriptions of the body~\citep{gallagher2006body,ataria2021body,sattin2023overview,hoffmann2010body}. Such representations are not rigidly tied to a single morphology; they can adapt under tool use~\citep{maravita2004tools,cardinali2009tool}, and studies of imitation suggest that action content can be partially dissociated from the specific limb used to realize it~\citep{chaminade2005fmri,caspers2010ale,van2011imitation}. 

In embodied AI, however, the action signal used to train world models is usually tied to a specific body and command interface. Recorded actions are imperfect proxies for the physical interventions that drive state change: control dynamics, contact, delays, and calibration differences mediate command execution, so the same nominal action can yield different transitions across embodiments or controllers~\citep{suomalainen2022survey,lynch2017modern}. 
This issue is especially problematic for cross-embodiment world modeling, where the shared intervention structure is easily obscured by substantial embodiment-specific realization details~\citep{o2024openX}. 
This raises a central question: can we condition world models on an action representation learned from visual transitions, one that captures how an agent intervenes in the world without being tied to a specific body or raw command space?

\begin{figure}[H]
    \centering
    \includegraphics[width=0.9\linewidth]{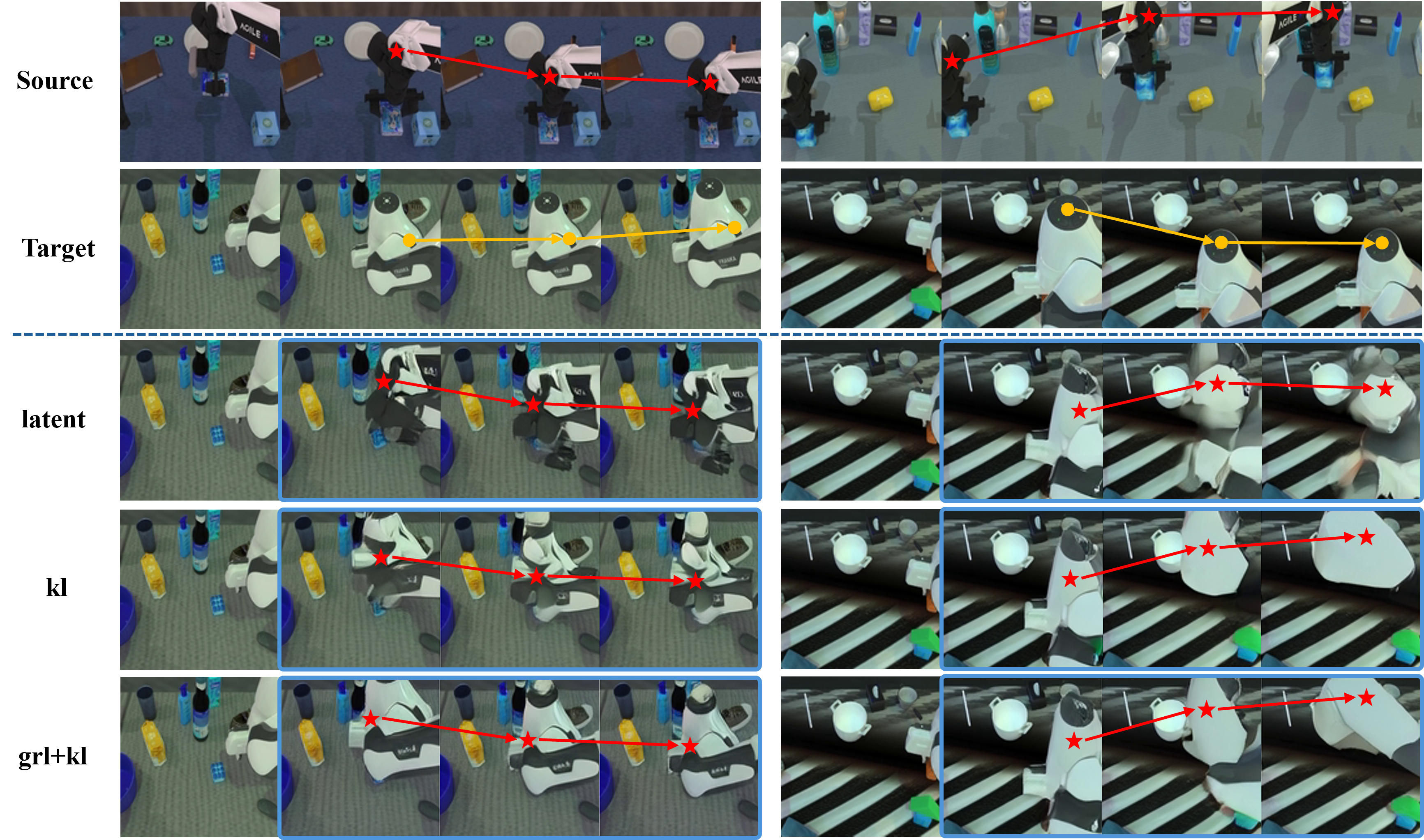}
\caption{\textbf{Cross-embodiment action transfer.}
Latent actions from a source ALOHA trajectory are applied to a target Franka context. 
KL limits visual shortcuts, and GRL reduces embodiment leakage; KL+GRL best preserves the target embodiment while transferring the source motion structure.}\label{fig:cross_embodiment_transfer}
\end{figure}
\vspace{-0.5cm}

Recent latent-action models~\citep{lapo,lapa, nikulinbackground,liang2025clam,garrido2026inthewild} learn action-like representations from visual transitions following an inverse-forward dynamics framework: an inverse dynamics model (IDM) infers latent action from visual transitions, while a forward dynamics model (FDM) is trained to reconstruct future observations conditioned on these latent actions. Many methods aim to unlock unlabeled video data for policy learning by learning compact action codes with little or no action supervision~\citep{lapo,lapa}. In contrast, our goal is to learn a unified action representation that can serve as a transferable conditioning interface for world models across embodied domains. However, learning such a representation from video is challenging because the latent may encode appearance, task identity, or embodiment-specific statistics instead of controllable change~\citep{zhanglatent}. 
It is further exacerbated when the FDM must learn visual dynamics from scratch, as reconstruction losses can dominate action abstraction; see Appendix~\ref{app:fdm_prior_ablation}.

We therefore propose \textbf{SCAR}, a self-supervised inverse--forward dynamics framework for learning \emph{unified continuous action representations across embodiments}. 
The key idea is to treat action as the part of a visual transition that explains controllable change, rather than as the embodiment-specific command used to produce it. 
A useful latent action should therefore retain the information needed for future prediction, while avoiding two shortcuts: it should not simply compress future visual details, and it should not encode embodiment-specific appearance or actuation statistics. The representation should also connect back to control, since a transition-side latent inferred from future observations is only useful at deployment if it can be recovered 
from raw commands and the current visual context.

SCAR implements these desiderata in the latent space of a pretrained causal VAE. 
The IDM infers a stochastic latent action $\bm{z}_t$ from a realized visual transition, while the FDM, initialized from the Wan video generative backbone~\citep{wan2025wan}, predicts future latent dynamics conditioned on $\bm{z}_t$, forming a latent-action-conditioned world model.
To prevent  above two shortcuts, we introduce KL regularization to limit the information capacity of the latent action posterior, and adversarial invariance via gradient reversal~\citep{ganin2015grl} to remove embodiment-identifying factors.
Because $\bm{z}_t$ is inferred from future observations during representation learning, we further train a lightweight context-conditioned action-to-latent controller that maps raw action sequences and visual context back into the learned latent action space. 
This separates representation discovery from embodiment-specific command realization: SCAR first learns what controllable change is shared across embodiments, and then learns how each embodiment realizes it through its own commands.

We evaluate the resulting latent action interface on Procgen \cite{cobbe2020procgen} and Robotwin \cite{chen2025robotwin}. Procgen tests environment transfer with a shared discrete action space, while Robotwin evaluates cross-embodiment transfer in continuous-control manipulation. Experiments show that the learned latent representation provides a stronger world-model conditioning interface than raw actions, improving low-data adaptation and transfer-task generalization. Fig.~\ref{fig:cross_embodiment_transfer} further shows that latent actions inferred from one embodiment can drive another embodiment's visual context, suggesting that SCAR learns transferable action structure across bodies. We also show that controllability is recovered more effectively through context-conditioned sequence-level action-to-latent prediction than through direct raw-action conditioning. Together, these results support a simple claim: the transferable part of action is not the embodiment-specific command, but the intervention structure that explains how the world changes, and this representation enables more generalizable world models.

\textbf{Contributions.}
Our contributions are threefold. 
(1) We formulate cross-embodiment world-model conditioning as a transition-side action representation problem, replacing embodiment-specific raw commands with a learned interface grounded in realized visual change. 
(2) We propose \textbf{SCAR}, an inverse--forward latent dynamics framework that learns compact and embodiment-invariant continuous latent actions from visual transitions. 
(3) We show that the learned latent interface improves low-data prediction and transfer, reduces embodiment leakage, and can be connected back to raw commands through an action-to-latent controller.
\section{Related Work}
\textbf{Action-conditioned world models.}
Robotic world models typically condition future prediction on recorded actions, either in image space~\cite{chi2025wow}, latent visual space~\cite{zhoudinowm}, or learned dynamics representations~\cite{maes2026leworldmodel,hafner2025dreamerv4}. Such action-conditioned models provide an interface for planning~\cite{maes2026leworldmodel,goswami2025dexwm} and policy learning~\cite{chandradiwa,bi2025motus,yang2026rise}, but recorded actions are tied to a specific embodiment, controller, and action parameterization~\cite{o2024openX}. SCAR instead treats the conditioning interface itself as a representation learning problem, using transition-side latent actions to condition a generative world model.

\textbf{Latent action learning from visual transitions.}
Recent latent-action methods infer action-like codes from visual transitions with an inverse dynamics model and train a forward model or video predictor to ensure that these codes explain future observations~\cite{lapo, lapa,nikulinbackground}. Many methods learn discrete action tokens, which can be limiting for fine-grained continuous manipulation. Continuous latent-action methods such as CLAM~\cite{liang2025clam}, VILA~\cite{jeong2026vila}, and VILLA-X~\cite{chen2025villa} introduce richer action spaces, but often rely on action labels, proprioception, or other supervision to anchor the latent. Other works scale latent-action learning with in-the-wild videos or VLA pretraining~\cite{garrido2026inthewild,luo2026JALA}. However, transition prediction alone can yield latents that encode appearance, task, or embodiment shortcuts rather than controllable change~\cite{lachapelle2025identifiability,zhanglatent}. SCAR addresses this issue through capacity regularization and embodiment-adversarial invariance, aiming to make the latent action more suitable for cross-embodiment world-model conditioning.

\textbf{Cross-embodiment robot learning.}
Large-scale robot learning must handle heterogeneous morphologies, sensors, and control spaces~\cite{o2024openX}. Existing approaches reduce this heterogeneity through multi-embodiment pretraining~\cite{intelligence2025pi05,bjorck2025gr00t, bauer2025latentdiffusion}, manually defined shared command abstractions such as human hand skeletons~\cite{yang2025egovla,li2025unidex,kareer2025egomimic}, or adversarial alignment of visual representations~\cite{cai2025InNOn,wang2024cross}. These methods primarily align policies, visual features, or predefined state--action interfaces. In contrast, SCAR learns the action-conditioning interface itself as a transition-side continuous latent representation, directly from realized visual dynamics across embodiments.
\section{Method}
\subsection{Problem Formulation: Learning a Shared Latent Action Interface}

We consider trajectories collected from multiple embodiments,
$\mathcal{D}=\{(\bm{x}_{1:T}, \bm{a}^{e}_{1:T-1}, e)\}$, where
$\bm{x}_{1:T}$ is a video sequence, $\bm{a}^{e}_{1:T-1}$ is the
embodiment-specific command sequence, and $e\in\mathcal{E}$ denotes the
embodiment identity. In cross-embodiment world modeling, raw commands are often
a poor common interface: the same command can induce different physical effects
under different morphologies, controllers, delays, and contact dynamics, while different actions may produce similar visual transitions. We therefore seek an inferred unified latent action $\bm{z}_t$ that is learned from realized visual transitions and captures
the controllable change shared across embodiments.

From a data-generation perspective, the key distinction is between the shared
effect of an action and the embodiment-specific way in which this effect is
realized. We write $\bm{u}_t$ for the unobserved \emph{ground-truth unified latent action} (the embodiment-invariant controllable factor in the data-generating process); $\bm{z}_t$ is the model's \emph{estimation} of $\bm{u}_t$, recovered from visual transitions and used to condition the world model. A raw command can be viewed as an embodiment-specific realization
$\bm{a}^{e}_t=h(\bm{u}_t,e)$, which drives the transition $\bm{s}_{t+1}=\mathcal{F}(\bm{s}_t,\bm{a}^{e}_t)$ and gives the
observation $\bm{x}_t=R(\bm{s}_t)$. An inverse dynamics model can therefore recover from
$(\bm{x}_t,\bm{x}_{t+1})$ an action coordinate $\tilde{\bm{a}}_t=I_\phi(\bm{x}_t,\bm{x}_{t+1})$ that explains
the realized change. However, since $\tilde{\bm{a}}_t$ is recovered through the
embodiment-specific transition, it may still contain embodiment-dependent bias.
The goal of SCAR is to keep the transition-predictive part of this coordinate
while removing the part that identifies the embodiment with a discriminative objective, yielding the inferred unified latent
action $\bm{z}_t$ for world-model conditioning.

In principle, SCAR gives the sufficient conditions for recovering unified latent actions: under a structured data-generation process, prediction can
recover an action coordinate from visual change, and adversarial embodiment
removal can select the shared controllable component. The full formal statement
and proof are provided in Appendix~\ref{app:proofs}.

\begin{theorem}[Informal: sufficient conditions for unified latent action recovery]
\label{thm:latent_action_recovery}
Suppose each trajectory is generated from embodiment-invariant ground-truth unified latent actions $\bm{u}_t$ through an embodiment-specific realization $\bm{a}^{e}_t=h(\bm{u}_t,e)$, followed
by $\bm{s}_{t+1}=\mathcal{F}(\bm{s}_t,\bm{a}^{e}_t)$ and $\bm{x}_t=R(\bm{s}_t)$. Assume these maps are injective on the data support. If the IDM recovers the raw action $\bm{a}^{e}_t$ up to an invertible reparameterization,
$\tilde{\bm{a}}_t=\rho(\bm{a}^{e}_t)$, and this recovered action coordinate can be decomposed
into a shared controllable component and an embodiment-specific component, then
an encoder trained with a population-level
discriminative objective removes the embodiment-specific component and
recovers the unified latent action space up to a per-embodiment invertible bijection
$\bm{z}_t = \Phi_e(\bm{u}_t)$. Therefore, $\bm{z}_t$ provides a unified latent action
interface for cross-embodiment world models.
\end{theorem}

Theorem~\ref{thm:latent_action_recovery} provides a data-generation view of a unified latent action. 
We instantiate this principle end-to-end in the latent space of a pretrained video model: the IDM infers a stochastic latent action sequence \(\bm{z}_{1:T-1}\) from latent video transitions, and the FDM predicts future latent dynamics conditioned on it. 
However, prediction alone is insufficient, since a predictive latent may still encode future visual details or embodiment-specific nuisance information. 
Thus, Sec.~\ref{sec:latent_action_learning} introduces the inverse--forward learning objective, while Sec.~\ref{sec:shortcut_latents} adds KL and GRL regularization: KL limits the capacity of \(\bm{z}_t\) to prevent unrestricted visual coding, and GRL discourages embodiment-discriminative information. 
Together, these components preserve transition-predictive controllable information while suppressing embodiment-specific shortcuts.

\subsection{Learning Predictive Latent Actions with an Inverse--Forward Model}
\label{sec:latent_action_learning}

As illustrated in Fig.~\ref{fig:framework}, SCAR operates in the latent space of a pretrained Wan video generation model~\citep{wan2025wan}. Given an input video window \(\bm{x}_{1:T}\), we first encode it with a frozen causal VAE,
$\bm{v}_{1:F} = E_{\mathrm{vae}}(\bm{x}_{1:T})$,
where \(\bm{v}_{1:F}\) denotes the temporally compressed latent representation. Operating in this latent space provides a structured representation of visual dynamics while avoiding the redundancy of raw pixels.

The inverse dynamics model infers a transition-level latent action sequence from the latent video:
\begin{equation}
    \bm{z}_{1:T-1} = I_{\phi}(\bm{v}_{1:F}).
\end{equation}
We parameterize the IDM output as a diagonal Gaussian posterior,
$q_{\phi}(\bm{z}_{1:T-1} \mid \bm{v}_{1:F})
=
\prod_{t=1}^{T-1}
\mathcal{N}\!\left(\bm{z}_t;\bm{\mu}_t,\mathrm{diag}(\bm{\sigma}_t^2)\right)$,
and sample \(\bm{z}_t = \bm{\mu}_t + \bm{\sigma}_t \odot \bm{\epsilon}_t\), with \(\bm{\epsilon}_t \sim \mathcal{N}(0,I)\), during training. At inference time, we use the posterior mean \(\bm{\mu}_t\).

The FDM predicts future latent dynamics conditioned on the inferred latent action. We map the frame-rate latent action sequence \(\bm{z}_{1:T-1}\) to a conditioning sequence aligned with the latent-video,
\begin{equation}
\bm{c}_{1:F} = A_{\psi}(\bm{z}_{1:T-1}),
\end{equation}
and inject \(\bm{c}_f\) into the pretrained Wan backbone through adaptive layer normalization:
\begin{equation}
(\bm{\beta}_f,\bm{\gamma}_f)=W_{\mathrm{mod}}\bm{c}_f,
\qquad
\texttt{AdaLN}(\bm{h};\bm{c}_f)=(1+\bm{\gamma}_f)\odot \texttt{LN}(\bm{h})+\bm{\beta}_f.
\end{equation}
This inverse--forward coupling forces the latent action to carry information that is useful for predicting the realized transition.

\begin{figure}[t]
    \centering
    \includegraphics[width=0.95\linewidth]{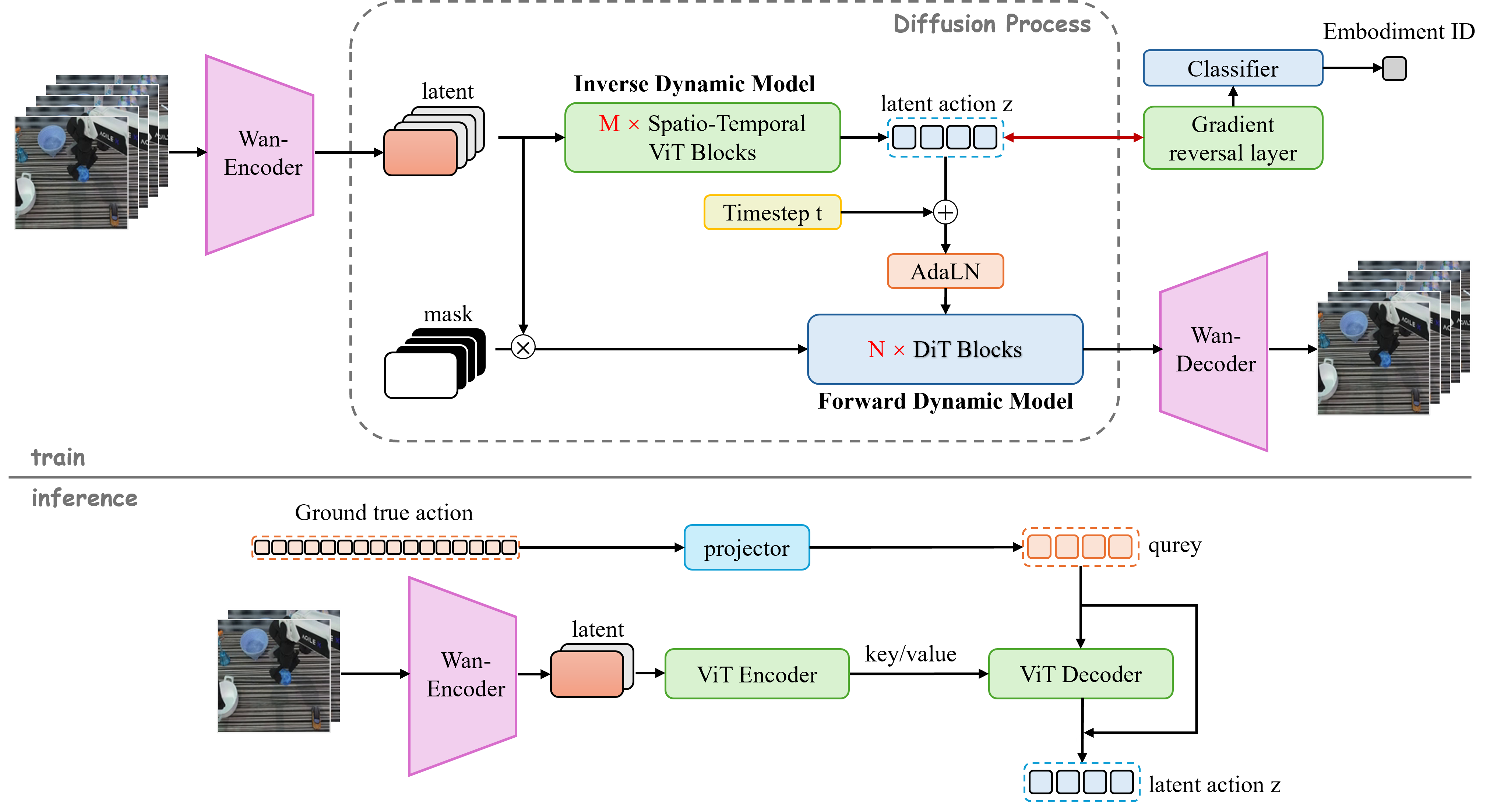}
    \caption{\textbf{Overview of SCAR}. 
During training, a frozen Wan encoder maps video frames into latent space; the IDM infers stochastic latent actions, and the FDM predicts future latent dynamics conditioned on them. 
KL constrains latent capacity and GRL suppresses embodiment-specific information. 
At inference time, an action-to-latent controller maps raw commands and visual context to latent actions for controllable prediction.}
    \label{fig:framework}
\end{figure}

\subsection{Preventing Shortcut Latents: Compactness and Embodiment Invariance}
\label{sec:shortcut_latents}

As suggested by Theorem~\ref{thm:latent_action_recovery}, prediction alone does
not guarantee that the latent action captures controllable change. In
cross-embodiment data, the IDM may reduce reconstruction loss by encoding
shortcuts such as embodiment identity, robot appearance, camera configuration,
or task-specific visual context. SCAR therefore constrains the latent action
with two complementary regularizers: a GRL objective that removes
embodiment-discriminative information, and a KL objective that limits the
capacity of the latent code.

\textbf{\textcolor{scarred}{Embodiment invariance via gradient reversal.}}
To suppress embodiment-specific shortcuts, we attach an adversarial embodiment classifier to the latent action. Each latent token is passed through a gradient reversal layer before classification, and the adversarial loss is defined as:
\begin{equation}
    \hat{e}_{i,t}
    = D_{\omega}(R_{\alpha}(\bm{z}_{i,t})),
\qquad
    \mathcal{L}_{\mathrm{GRL}}
    =
    \frac{1}{|\Omega|}
    \sum_{(i,t)\in\Omega}
    \mathrm{CE}(\hat{e}_{i,t},e_i).
\end{equation}
The classifier is trained to predict the embodiment identity, while the IDM
receives the reversed gradient and is encouraged to remove
embodiment-discriminative information from $\bm{z}_t$. Appendix~\ref{app:proofs} provides a
sufficient-condition analysis showing that, under idealized structured-nuisance
assumptions and a non-collapsed predictive encoder, the adversarial objective
removes embodiment-dependent directions while preserving a transition-predictive
latent subspace.

\textbf{\textcolor{scarblue}{Compactness via KL regularization.}}
To prevent $\bm{z}_t$ from becoming an unconstrained visual compression code, we
regularize the posterior toward a standard Gaussian prior:
\begin{equation}
    \mathcal{L}_{\mathrm{KL}}
    =
    \frac{1}{T-1}
    \sum_{t=1}^{T-1}
    D_{\mathrm{KL}}
    \Big(
    q_{\phi}(\bm{z}_t \mid \bm{v}_{1:F})
    \,\|\, \mathcal{N}(0,I)
    \Big).
\end{equation}
This limits the capacity of the latent action and makes it harder to hide
nuisance visual or embodiment-specific details in a high-dimensional code.

\textbf{Why the two regularizers are complementary.}
For a fixed latent-action encoder, the optimal embodiment classifier in the GRL
branch satisfies
\begin{equation}
    \min_{\omega} \mathcal{L}_{\mathrm{GRL}}(\omega,\phi)
    =
    H(e\mid \bm{z})
    =
    H(e)-I(e;\bm{z}).
\end{equation}
Thus, reversing the gradient of $L_{\mathrm{GRL}}$ encourages the encoder to
reduce the mutual information between the latent action $\bm{z}_t$ and the embodiment
identity $e$. However, adversarial invariance alone is not sufficient: a
collapsed latent would also make $I(e;\bm{z})$ small. The prediction loss prevents
this trivial solution by forcing $\bm{z}_t$ to explain the realized transition, while
$L_{\mathrm{KL}}$ limits latent capacity and $L_{\mathrm{GRL}}$ removes
embodiment-discriminative information.

In summary, the three pressures play distinct roles:
\textcolor{scargreen}{\textbf{prediction }}(Sec.~\ref{sec:latent_action_learning}) keeps $\bm{z}_t$ transition-relevant,
\textcolor{scarblue}{\textbf{KL}} makes $\bm{z}_t$ compact, and
\textcolor{scarred}{\textbf{GRL}} makes $\bm{z}_t$ embodiment-invariant. Together,
they encourage a latent action that captures controllable change rather than
visual or embodiment-specific shortcuts. 
\subsection{Recovering Controllability from Raw Commands}

The transition-side latent action is inferred from visual changes during representation learning. At inference time, however, a controllable world model must be driven by raw command sequences. We therefore learn a context-conditioned action-to-latent controller that maps raw actions and visual context to the learned latent action space:
\begin{equation}
\bm{h}_{1:F_{\mathrm{ctx}}}
=
E_{\xi}\!\left(\bm{v}_{1:F_{\mathrm{ctx}}}\right),
\qquad
\hat{\bm{z}}_{1:T-1}
=
D_{\eta}\!\left(\bm{a}_{1:T-1}, \bm{h}_{1:F_{\mathrm{ctx}}}\right),
\end{equation}
where \(\bm{v}_{1:F_{\mathrm{ctx}}}\) denotes observed context latents, \(E_{\xi}\) encodes them as key-value memory, and \(D_{\eta}\) is a transformer decoder that internally projects raw actions as query inputs to predict latent actions.

The controller is trained to match the frozen IDM latent actions:
\begin{equation}
\mathcal{L}_{\mathrm{A2L}}
=
\frac{1}{|\Omega|}
\sum_{t\in\Omega}
\left\|
\hat{\bm{z}}_t - \mathrm{sg}(\bm{\mu}_{\phi,t})
\right\|_2^2 ,
\end{equation}
where \(\bm{\mu}_{\phi,t}\) is the posterior mean inferred by the IDM and \(\mathrm{sg}(\cdot)\) denotes stop-gradient.

This sequence-level design avoids requiring pointwise alignment between raw commands and latent actions. Such alignment can be unstable because raw commands are mediated by embodiment-specific kinematics, controllers, delays, and contact dynamics. Sec.~\ref{sec:action_interface_comparison} empirically compares this sequence-level interface with pointwise action-to-latent prediction and direct raw-action conditioning.
\subsection{Training Objective}

Given a latent video tensor \(\bm{v}\), we sample a diffusion timestep \(\tau\) and noise \(\bm{\epsilon}\), and form
\begin{equation}
\tilde{\bm{v}}_{\tau}=(1-\sigma_{\tau})\bm{v}+\sigma_{\tau}\bm{\epsilon},
\qquad
\bm{u}_{\tau}=\bm{\epsilon}-\bm{v}.
\end{equation}
The action-conditioned FDM predicts the flow-matching velocity
\(\hat{\bm{u}}_{\tau}=g_{\theta}(\tilde{\bm{v}}_{\tau},\tau,\bm{z}_{1:T-1})\),
where \(\bm{u}_{\tau}\) denotes the diffusion flow target and is distinct from the unified latent action in Sec.~\ref{sec:latent_action_learning}. We train it with
$\mathcal{L}_{\mathrm{rec}}
=
\mathbb{E}_{\bm{x}_{1:T},\tau,\bm{\epsilon}}
\left[
\|\hat{\bm{u}}_{\tau}-\bm{u}_{\tau}\|_2^2
\right]$.
The overall objective is
\begin{equation}
\mathcal{L}_{\mathrm{total}}
=
\mathcal{L}_{\mathrm{rec}}
+
\beta \mathcal{L}_{\mathrm{KL}}
+
\lambda_{\mathrm{adv}}\mathcal{L}_{\mathrm{GRL}},
\end{equation}
where \(\mathcal{L}_{\mathrm{rec}}\) enforces predictive sufficiency, \(\mathcal{L}_{\mathrm{KL}}\) encourages compactness, and \(\mathcal{L}_{\mathrm{GRL}}\) discourages embodiment-specific shortcuts.







\section{Experiments}
\subsection{Benchmarks and Protocols}

We evaluate SCAR on two embodied world-modeling benchmarks: Robotwin for robotic cross-embodiment manipulation and Procgen for virtual-embodiment transfer in procedurally generated control environments. Robotwin evaluates transfer across heterogeneous robot morphologies and continuous control interfaces, while Procgen tests whether the same latent-action principle applies to visually and dynamically diverse virtual agents under a shared discrete action interface.

\textbf{Robotwin.}
We use \texttt{place\_a2b\_left} from Robotwin~\citep{chen2025robotwin} as the main robotic manipulation task. We consider four embodiments:$\mathcal{E}=\{\texttt{aloha-agilex},\ \texttt{arx-x5},\ \texttt{franka},\ \texttt{ur5}\}.$
Franka is used as the low-data target embodiment with only \(m=10\) target training episodes, while the remaining three embodiments provide 300 source episodes each. Models are trained on \texttt{place\_a2b\_left}. We evaluate both target-task prediction on held-out \texttt{place\_a2b\_left} episodes and zero-shot transfer-task prediction on \texttt{place\_a2b\_right}.

\textbf{Procgen.}
We evaluate virtual-embodiment transfer on Procgen~\citep{cobbe2020procgen}. Although Procgen environments share a discrete action interface, they differ in agent appearance, motion dynamics, interaction rules, layouts, and task structure. We consider two groups: $\mathcal{G}_1=\{\texttt{caveflyer},\ \texttt{chaser},\ \texttt{ninja}\}, 
\quad
\mathcal{G}_2=\{\texttt{heist},\ \texttt{jumper},\ \texttt{miner}\}$.
For each group, we use leave-one-environment-out low-data adaptation: one environment is held out as the target with \(m=10\) training episodes and 50 held-out evaluation episodes, while the remaining environments provide 300 source episodes each.

\textbf{Implementation and metrics.}
For both benchmarks, we encode \(T=49\)-frame video windows with a frozen Wan causal VAE and predict \(32\) future frames. The IDM is a \(6\)-layer spatiotemporal transformer with latent action dimension \(d_z=64\), and the FDM is initialized from pretrained Wan2.1 1.3B weights and conditioned on latent actions through AdaLN modulation. Models are trained with AdamW for \(10{,}000\) steps using an effective global batch size of \(16\); additional hyperparameters and temporal-alignment details are provided in Appendix~\ref{app:implementation_details}. For Robotwin, raw actions are padded to a shared \(16\)-dimensional interface; for Procgen, raw-action baselines use one-hot discrete actions. We report SSIM, PSNR, MSE, and SSIM-L, where SSIM-L denotes last-frame SSIM.

\subsection{Latent-Interface Transfer Evaluation}
We first evaluate whether the learned latent action provides a better conditioning interface for world modeling. For latent-action methods, the IDM infers latent actions from evaluation transitions, and the FDM predicts future dynamics conditioned on these inferred latents. This setting is a representation diagnostic: it measures whether the latent action explains realized transitions and transfers across domains. It is not used as a deployable control interface; controllability from raw actions is evaluated separately in Sec.~\ref{sec:action_interface_comparison}.

\textbf{Compared methods.}
We compare raw-action and latent-action conditioning under target-only and shared-data settings. \textbf{Target-Only-GT} uses only target-domain data with ground-truth actions, while \textbf{Shared-GT} uses both source and target data with ground-truth actions. \textbf{Target-Only-Latent} and \textbf{Shared-Latent} use IDM-inferred latent actions under the same data settings. We further evaluate three SCAR variants: \textbf{SCAR-kl} with KL regularization, \textbf{SCAR-grl} with GRL-based adversarial invariance, and \textbf{SCAR-kl-grl} with both objectives.


\begin{table}[t]
\centering
\small
\setlength{\tabcolsep}{4.2pt}
\caption{Transfer evaluation across embodied domains under the low-data target setting (\(m=10\)). 
(a) Procgen evaluates embodiment transfer across generated control environments. 
(b) Robotwin evaluates robotic cross-embodiment transfer, with target-task evaluation on \texttt{place\_a2b\_left} and zero-shot transfer-task evaluation on \texttt{place\_a2b\_right}. All results are averaged over three seeds.} 
\label{tab:transfer_eval_all}
\vspace{0.3em}
\textbf{(a) Procgen virtual-embodiment transfer.}
\vspace{0.3em}

\begin{tabular}{lcccccccc}
\toprule
& \multicolumn{4}{c}{\(\mathcal{G}_1\): \texttt{caveflyer/chaser/ninja}}
& \multicolumn{4}{c}{\(\mathcal{G}_2\): \texttt{heist/jumper/miner}} \\
\cmidrule(lr){2-5} \cmidrule(lr){6-9}
Method
& SSIM \(\uparrow\) & PSNR \(\uparrow\) & MSE \(\downarrow\) & SSIM-L \(\uparrow\)
& SSIM \(\uparrow\) & PSNR \(\uparrow\) & MSE \(\downarrow\) & SSIM-L \(\uparrow\) \\
\midrule
Target-Only-GT
& 0.421 & 13.28 & 0.0476 & 0.356
& 0.374 & 12.74 & 0.0562 & 0.302 \\
Shared-GT
& 0.493 & 14.06 & 0.0398 & 0.424
& 0.451 & 13.52 & 0.0469 & 0.372 \\
Target-Only-Latent
& 0.448 & 13.61 & 0.0439 & 0.381
& 0.398 & 13.01 & 0.0523 & 0.327 \\
Shared-Latent
& 0.565 & 15.02 & 0.0321 & 0.505
& 0.526 & 14.46 & 0.0374 & 0.458 \\
SCAR-kl
& 0.579 & 15.18 & 0.0309 & 0.522
& 0.545 & 14.74 & 0.0352 & 0.481 \\
SCAR-grl
& 0.572 & 15.09 & 0.0315 & 0.514
& 0.536 & 14.61 & 0.0361 & 0.470 \\
SCAR-kl-grl
& \textbf{0.594} & \textbf{15.37} & \textbf{0.0293} & \textbf{0.538}
& \textbf{0.563} & \textbf{15.03} & \textbf{0.0329} & \textbf{0.506} \\
\bottomrule
\end{tabular}

\vspace{0.3em}

\textbf{(b) Robotwin robotic cross-embodiment transfer.}
\vspace{0.3em}

\begin{tabular}{lcccccccc}
\toprule
& \multicolumn{4}{c}{Target task} 
& \multicolumn{4}{c}{Transfer task} \\
\cmidrule(lr){2-5} \cmidrule(lr){6-9}
Method 
& SSIM \(\uparrow\) & PSNR \(\uparrow\) & MSE \(\downarrow\) & SSIM-L \(\uparrow\)
& SSIM \(\uparrow\) & PSNR \(\uparrow\) & MSE \(\downarrow\) & SSIM-L \(\uparrow\) \\
\midrule
Target-Only-GT     
& 0.536 & 15.69 & 0.0270 & 0.511
& 0.573 & 16.47 & 0.0226 & 0.543 \\
Shared-GT          
& 0.713 & 16.70 & 0.0214 & 0.670
& 0.731 & 17.14 & 0.0193 & 0.683 \\
Target-Only-Latent 
& 0.536 & 16.18 & 0.0241 & 0.508
& 0.581 & 17.01 & 0.0199 & 0.544 \\
Shared-Latent      
& 0.743 & 17.99 & 0.0159 & 0.721
& 0.756 & 18.26 & 0.0149 & 0.729 \\
SCAR-kl            
& 0.752 & 18.27 & 0.0149 & 0.727
& 0.763 & 18.51 & 0.0141 & 0.738 \\
SCAR-grl           
& 0.745 & 18.01 & 0.0158 & 0.720
& 0.761 & 18.42 & 0.0144 & 0.733 \\
SCAR-kl-grl        
& \textbf{0.759} & \textbf{18.49} & \textbf{0.0142} & \textbf{0.735}
& \textbf{0.770} & \textbf{18.70} & \textbf{0.0135} & \textbf{0.747} \\
\bottomrule
\end{tabular}
\end{table}

\begin{figure}[t]
    \centering
    \includegraphics[width=0.9\linewidth]{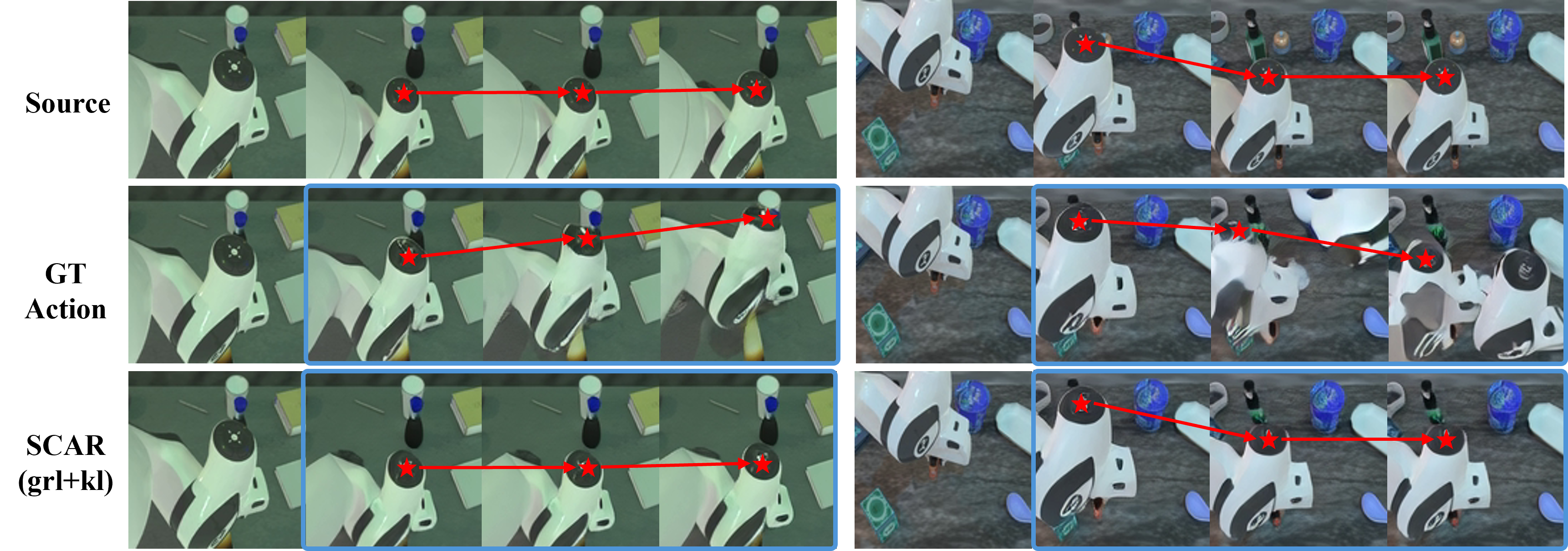}
    \caption{Qualitative zero-shot cross-task generation on the target embodiment. Under task shift, the GT-action-conditioned model produces hallucinated robot structures, while SCAR with KL+GRL better preserves the target embodiment and follows the transferred action trajectory.}
    \label{fig:cross_task_qualitative}
\end{figure}

\begin{figure}[t]
\centering
\begin{minipage}[t]{0.48\linewidth}
    \centering
    \includegraphics[width=0.9\linewidth]{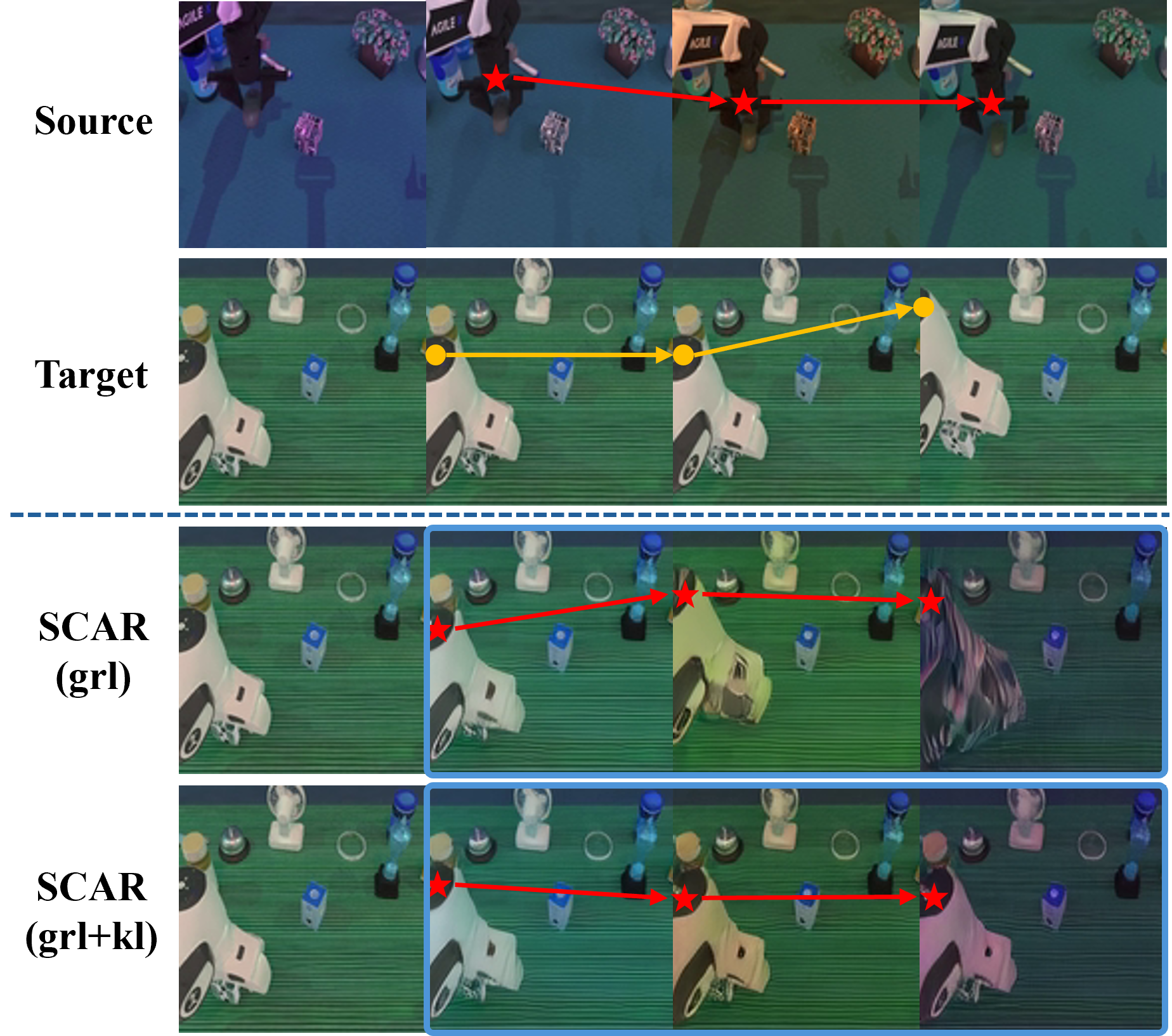}
    \caption{Qualitative ablation under nuisance visual variation from changing lighting. GRL alone retains scene variation, while KL regularization suppresses redundant appearance cues and improves cross-embodiment transfer.}
    \label{fig:ablation_KL}
\end{minipage}
\hfill
\begin{minipage}[t]{0.48\linewidth}
    \centering
    \includegraphics[width=\linewidth]{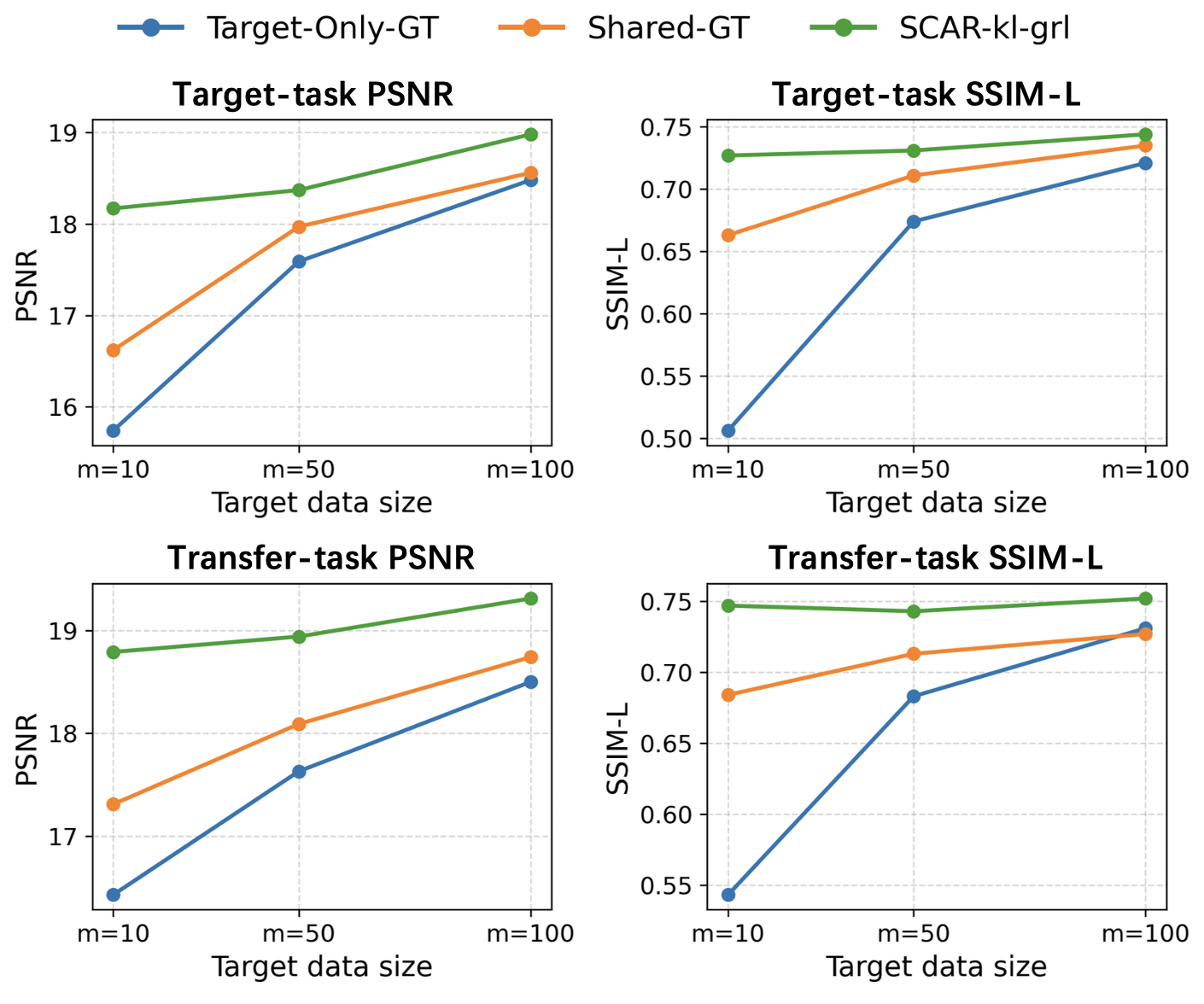}
    \caption{Performance trends across target-embodiment data sizes (\(m=10,50,100\)) on Robotwin target and transfer tasks. SCAR-kl-grl consistently outperforms raw-action baselines across data regimes.}
    \label{fig:data_size_curve}
\end{minipage}
\end{figure}

Table~\ref{tab:transfer_eval_all} shows two consistent trends. First, latent-action conditioning improves over raw-action conditioning in both target-only and shared-data settings, indicating that transition-side latents provide a more effective world-model interface than embodiment-specific commands. Second, KL and GRL regularization further improve the shared latent baseline, with \textbf{SCAR-kl-grl} performing best across all Procgen and Robotwin settings. These gains are moderate in reconstruction metrics but consistent across domains, tasks, and target-data regimes, suggesting that the main benefit is not merely better fitting, but a more transferable conditioning signal. 

Fig.~\ref{fig:cross_task_qualitative} provides a qualitative view of the same effect. Under cross-task transfer, raw-action conditioning is more prone to hallucinating embodiment structure, whereas SCAR better preserves the target robot while following the transferred motion pattern. Fig.~\ref{fig:ablation_KL} further illustrates the role of regularization: without sufficient capacity control, the latent can retain nuisance appearance variation, while KL+GRL produces cleaner transfer. Fig.~\ref{fig:data_size_curve} shows that the advantage of SCAR-kl-grl persists as the number of target-embodiment episodes increases.

\textbf{Embodiment leakage after action transfer.}
We next directly measure whether transferred rollouts preserve the target embodiment. Latent actions are extracted from source embodiments \texttt{aloha}, \texttt{arx}, and \texttt{ur5}, then applied to a \texttt{franka} visual context. We apply an independent single-frame embodiment classifier to predicted future frames only. As shown in Table~\ref{tab:action_transfer_leakage}, \textbf{SCAR-kl-grl} reduces source-embodiment probability and improves target preservation compared with \textbf{Shared-Latent}. This supports the intended role of KL and GRL: reducing embodiment leakage while retaining transition-relevant action structure. Details are provided in Appendix~\ref{app:action_transfer_leakage}.

\begin{table}[t]
\centering
\small
\setlength{\tabcolsep}{6pt}
\caption{Embodiment leakage after cross-embodiment action transfer. Source latent actions are extracted from \texttt{aloha}, \texttt{arx}, and \texttt{ur5} and applied to a \texttt{franka} context, with results averaged over source embodiments.}
\label{tab:action_transfer_leakage}
\begin{tabular}{lcccc}
\toprule
Method 
& Source Prob \(\downarrow\)
& Target Prob \(\uparrow\)
& Target Share \(\uparrow\)
& Target-Source \(\uparrow\)\\
\midrule
Shared-Latent 
& 0.1020 & 0.8105 & 0.8896 & 0.7085  \\
SCAR-kl-grl 
& \textbf{0.0736} & \textbf{0.8333} & \textbf{0.9231} & \textbf{0.7598}  \\
\bottomrule
\end{tabular}
\end{table}
\subsection{Action Interface Comparison on the Target Embodiment}
\label{sec:action_interface_comparison}
The previous experiment uses posterior latents to diagnose representation quality. We now evaluate whether the learned latent interface can be connected back to raw actions. On the Franka target embodiment, we compare action interfaces under the \(m=50\) setting using both the held-out target-task split \texttt{place\_a2b\_left} and the transfer-task split \texttt{place\_a2b\_right}.

\textbf{Compared methods.}
\textbf{Target-Only-GT} and \textbf{Shared-GT} condition the world model directly on raw actions, using target-only data or multi-embodiment data, respectively. \textbf{Pointwise-A2L} predicts latent actions independently from raw actions at each timestep, while \textbf{Sequence-A2L} predicts the full latent action sequence from raw actions and target visual context. \textbf{Sequence-A2L-FT} further fine-tunes the sequence-level A2L controller and FDM while keeping the IDM fixed. All A2L variants use target latent actions extracted from the frozen \textbf{SCAR-kl-grl} IDM.

\begin{table}[t]
\centering
\small
\setlength{\tabcolsep}{4pt}
\caption{Action-interface comparison on Franka under \(m=50\)}
\label{tab:action_interface_comparison} 
\begin{tabular}{lcccccccc}
\toprule
& \multicolumn{4}{c}{Target task} 
& \multicolumn{4}{c}{Transfer task} \\
\cmidrule(lr){2-5} \cmidrule(lr){6-9}
Method 
& SSIM \(\uparrow\) & PSNR \(\uparrow\) & MSE \(\downarrow\) & SSIM-L \(\uparrow\)
& SSIM \(\uparrow\) & PSNR \(\uparrow\) & MSE \(\downarrow\) & SSIM-L \(\uparrow\) \\
\midrule
Target-Only-GT 
& 0.697 & 17.59 & 0.0174 & 0.674
& 0.708 & 17.63 & 0.0173 & 0.683 \\
Shared-GT 
& 0.746 & 17.97 & 0.0159 & 0.711
& 0.758 & 18.09 & 0.0155 & 0.713 \\
Pointwise-A2L 
& 0.676 & 15.96 & 0.0254 & 0.697
& 0.709 & 16.42 & 0.0228 & 0.691 \\
Sequence-A2L 
& 0.741 & 17.80 & 0.0166 & 0.712
& 0.748 & 17.92 & 0.0161 & 0.718 \\
Sequence-A2L-FT 
& \textbf{0.768} & \textbf{18.78} & \textbf{0.0132} & \textbf{0.741}
& \textbf{0.768} & \textbf{18.62} & \textbf{0.0137} & \textbf{0.731} \\
\bottomrule
\end{tabular}
\end{table}

As shown in Table~\ref{tab:action_interface_comparison}, sequence-level A2L substantially outperforms pointwise A2L, indicating that transition-side latent actions are better recovered as temporally contextual sequences rather than independent per-step targets. Without updating the FDM, \textbf{Sequence-A2L} already approaches \textbf{Shared-GT}. With lightweight interface adaptation, \textbf{Sequence-A2L-FT} reaches or exceeds \textbf{Shared-GT} on both target and transfer tasks. These results support a factorized interface: SCAR learns a transferable transition-side latent, while A2L learns the target embodiment's command-to-latent realization.
\subsection{Action Probe}

Finally, we test whether the learned latent action still preserves action-relevant information. We freeze the IDM and train a lightweight sequence regressor to decode raw action sequences from IDM-extracted latent actions using an \(\ell_2\) regression loss. The probe is trained on target-embodiment training trajectories and evaluated on held-out target trajectories.

\vspace{-.5em} 
\begin{table}[h]
\centering
\small
\setlength{\tabcolsep}{4.5pt}
\renewcommand{\arraystretch}{0.95}
\caption{Frozen action probe results on the target embodiment. Lower is better for all metrics.}
\label{tab:action_probe_results}
\begin{tabular}{lcccc}
\toprule
Model & Train MSE \(\downarrow\) & Train L1 \(\downarrow\) & Eval MSE \(\downarrow\) & Eval L1 \(\downarrow\) \\
\midrule
Shared-Latent & \textbf{0.0766} & \textbf{0.193} & 0.149 & 0.289 \\
SCAR-kl & 0.0797 & 0.197 & 0.140 & 0.276 \\
SCAR-grl & 0.0896 & 0.213 & 0.163 & 0.306 \\
SCAR-kl-grl & 0.0789 & 0.198 & \textbf{0.134} & \textbf{0.273} \\
\bottomrule
\end{tabular}
\end{table}

Table~\ref{tab:action_probe_results}  shows that Shared-Latent obtains the lowest training error, while SCAR-kl-grl achieves the best held-out MSE and L1. Thus, regularization does not discard action information; it reduces training-set-specific shortcuts while preserving action-relevant structure.


\section{Conclusion and Discussion}

SCAR moves beyond the default assumption that a world model must be conditioned on the raw command space of a particular embodiment.
Instead, it treats action as a latent factor of visual change: the part of a transition that explains how an agent intervenes in the world, separated from the embodiment-specific command used to realize that intervention. 
By learning this factor through inverse-forward dynamics, and constraining it to be compact and less embodiment-discriminative, SCAR provides a transferable action interface for pretrained video dynamics models. 
The gains across Procgen and Robotwin, together with reduced embodiment leakage, suggest that visual transitions can reveal shared action structure without action supervision for latent discovery.

This work is therefore a step toward world models whose action interface is defined by controllable change rather than robot-specific commands. 
Posterior-latent evaluation isolates the representation-learning problem, while A2L shows that the learned latent space can be connected back to raw commands. 
Scaling to larger real-robot datasets, in-the-wild videos, and closed-loop policy learning remains an important next step. 
More broadly, SCAR suggests that transferable world models require not only better visual prediction, but also better abstractions of action itself.



\bibliography{references}

@article{wan2025wan,
  title={Wan: Open and advanced large-scale video generative models},
  author={Wan, Team and Wang, Ang and Ai, Baole and Wen, Bin and Mao, Chaojie and Xie, Chen-Wei and Chen, Di and Yu, Feiwu and Zhao, Haiming and Yang, Jianxiao and others},
  journal={arXiv preprint arXiv:2503.20314},
  year={2025}
}

@article{liang2025clam,
  title={Clam: Continuous latent action models for robot learning from unlabeled demonstrations},
  author={Liang, Anthony and Czempin, Pavel and Hong, Matthew and Zhou, Yutai and Biyik, Erdem and Tu, Stephen},
  journal={arXiv preprint arXiv:2505.04999},
  year={2025}
}

@book{gallagher2006body,
  title={How the body shapes the mind},
  author={Gallagher, Shaun},
  year={2006},
  publisher={Clarendon press}
}

@book{ataria2021body,
  title={Body schema and body image: New directions},
  author={Ataria, Yochai and Tanaka, Shogo and Gallagher, Shaun},
  year={2021},
  publisher={Oxford University Press}
}

@article{sattin2023overview,
  title={An overview of the body schema and body image: theoretical models, methodological settings and pitfalls for rehabilitation of persons with neurological disorders},
  author={Sattin, Davide and Parma, Chiara and Lunetta, Christian and Zulueta, Aida and Lanzone, Jacopo and Giani, Luca and Vassallo, Marta and Picozzi, Mario and Parati, Eugenio Agostino},
  journal={Brain Sciences},
  volume={13},
  number={10},
  pages={1410},
  year={2023},
  publisher={MDPI}
}

@article{hoffmann2010body,
  title={Body schema in robotics: a review},
  author={Hoffmann, Matej and Marques, Hugo and Arieta, Alejandro and Sumioka, Hidenobu and Lungarella, Max and Pfeifer, Rolf},
  journal={IEEE Transactions on Autonomous Mental Development},
  volume={2},
  number={4},
  pages={304--324},
  year={2010},
  publisher={IEEE}
}

@article{maravita2004tools,
  title={Tools for the body (schema)},
  author={Maravita, Angelo and Iriki, Atsushi},
  journal={Trends in cognitive sciences},
  volume={8},
  number={2},
  pages={79--86},
  year={2004},
  publisher={Elsevier}
}

@article{cardinali2009tool,
  title={Tool-use induces morphological updating of the body schema},
  author={Cardinali, Lucilla and Frassinetti, Francesca and Brozzoli, Claudio and Urquizar, Christian and Roy, Alice C and Farn{\`e}, Alessandro},
  journal={Current biology},
  volume={19},
  number={12},
  pages={R478--R479},
  year={2009},
  publisher={Elsevier}
}

@article{chaminade2005fmri,
  title={An fMRI study of imitation: action representation and body schema},
  author={Chaminade, Thierry and Meltzoff, Andrew N and Decety, Jean},
  journal={Neuropsychologia},
  volume={43},
  number={1},
  pages={115--127},
  year={2005},
  publisher={Elsevier}
}

@article{caspers2010ale,
  title={ALE meta-analysis of action observation and imitation in the human brain},
  author={Caspers, Svenja and Zilles, Karl and Laird, Angela R and Eickhoff, Simon B},
  journal={Neuroimage},
  volume={50},
  number={3},
  pages={1148--1167},
  year={2010},
  publisher={Elsevier}
}

@article{van2011imitation,
  title={Imitation of hand and tool actions is effector-independent},
  author={Van Elk, M and Van Schie, HT and Bekkering, H},
  journal={Experimental brain research},
  volume={214},
  number={4},
  pages={539--547},
  year={2011},
  publisher={Springer}
}

@article{maes2026leworldmodel,
  title={Leworldmodel: Stable end-to-end joint-embedding predictive architecture from pixels},
  author={Maes, Lucas and Lidec, Quentin Le and Scieur, Damien and LeCun, Yann and Balestriero, Randall},
  journal={arXiv preprint arXiv:2603.19312},
  year={2026}
}

@inproceedings{zhoudinowm,
  title={DINO-WM: World Models on Pre-trained Visual Features enable Zero-shot Planning},
  author={Zhou, Gaoyue and Pan, Hengkai and LeCun, Yann and Pinto, Lerrel},
  booktitle={Forty-second International Conference on Machine Learning}
}

@article{chi2025wow,
  title={Wow: Towards a world omniscient world model through embodied interaction},
  author={Chi, Xiaowei and Jia, Peidong and Fan, Chun-Kai and Ju, Xiaozhu and Mi, Weishi and Zhang, Kevin and Qin, Zhiyuan and Tian, Wanxin and Ge, Kuangzhi and Li, Hao and others},
  journal={arXiv preprint arXiv:2509.22642},
  year={2025}
}

@article{goswami2025dexwm,
  title={World Models Can Leverage Human Videos for Dexterous Manipulation},
  author={Goswami, Raktim Gautam and Bar, Amir and Fan, David and Yang, Tsung-Yen and Zhou, Gaoyue and Krishnamurthy, Prashanth and Rabbat, Michael and Khorrami, Farshad and LeCun, Yann},
  journal={arXiv preprint arXiv:2512.13644},
  year={2025}
}

@inproceedings{chandradiwa,
  title={DiWA: Diffusion Policy Adaptation with World Models},
  author={Chandra, Akshay L and Nematollahi, Iman and Huang, Chenguang and Welschehold, Tim and Burgard, Wolfram and Valada, Abhinav},
  booktitle={9th Annual Conference on Robot Learning}
}

@inproceedings{o2024openX,
  title={Open x-embodiment: Robotic learning datasets and rt-x models: Open x-embodiment collaboration 0},
  author={O’Neill, Abby and Rehman, Abdul and Maddukuri, Abhiram and Gupta, Abhishek and Padalkar, Abhishek and Lee, Abraham and Pooley, Acorn and Gupta, Agrim and Mandlekar, Ajay and Jain, Ajinkya and others},
  booktitle={2024 IEEE International Conference on Robotics and Automation (ICRA)},
  pages={6892--6903},
  year={2024},
  organization={IEEE}
}

@inproceedings{lapa,
  title={Latent Action Pretraining from Videos},
  author={Ye, Seonghyeon and Jang, Joel and Jeon, Byeongguk and Joo, Se June and Yang, Jianwei and Peng, Baolin and Mandlekar, Ajay and Tan, Reuben and Chao, Yu-Wei and Lin, Bill Yuchen and others},
  booktitle={The Thirteenth International Conference on Learning Representations}
}

@inproceedings{lapo,
  title={Learning to Act without Actions},
  author={Schmidt, Dominik and Jiang, Minqi},
  booktitle={The Twelfth International Conference on Learning Representations}
}

@article{bauer2025latentdiffusion,
  title={Latent action diffusion for cross-embodiment manipulation},
  author={Bauer, Erik and Nava, Elvis and Katzschmann, Robert K},
  journal={arXiv preprint arXiv:2506.14608},
  year={2025}
}

@article{garrido2026inthewild,
  title={Learning Latent Action World Models In The Wild},
  author={Garrido, Quentin and Nagarajan, Tushar and Terver, Basile and Ballas, Nicolas and LeCun, Yann and Rabbat, Michael},
  journal={arXiv preprint arXiv:2601.05230},
  year={2026}
}

@article{chen2025villa,
  title={Villa-x: enhancing latent action modeling in vision-language-action models},
  author={Chen, Xiaoyu and Wei, Hangxing and Zhang, Pushi and Zhang, Chuheng and Wang, Kaixin and Guo, Yanjiang and Yang, Rushuai and Wang, Yucen and Xiao, Xinquan and Zhao, Li and others},
  journal={arXiv preprint arXiv:2507.23682},
  year={2025}
}

@article{lachapelle2025identifiability,
  title={On the Identifiability of Latent Action Policies},
  author={Lachapelle, S{\'e}bastien},
  journal={arXiv preprint arXiv:2510.01337},
  year={2025}
}

@inproceedings{zhanglatent,
  title={What Do Latent Action Models Actually Learn?},
  author={Zhang, Chuheng and Pearce, Tim and Zhang, Pushi and Wang, Kaixin and Chen, Xiaoyu and Shen, Wei and Zhao, Li and Bian, Jiang},
  booktitle={The Thirty-ninth Annual Conference on Neural Information Processing Systems}
}

@inproceedings{nikulinbackground,
  title={Latent Action Learning Requires Supervision in the Presence of Distractors},
  author={Nikulin, Alexander and Zisman, Ilya and Tarasov, Denis and Nikita, Lyubaykin and Polubarov, Andrei and Kiselev, Igor and Kurenkov, Vladislav},
  booktitle={Forty-second International Conference on Machine Learning}
}

@article{intelligence2025pi05,
  title={{$\pi_0.5$}: A Vision-Language-Action Model with Open-World Generalization},
  author={{Physical Intelligence} and Black, Kevin and Brown, Noah and Darpinian, James and Dhabalia, Karan and Driess, Danny and Esmail, Adnan and Equi, Michael and Finn, Chelsea and Fusai, Niccolo and others},
  journal={arXiv preprint arXiv:2504.16054},
  year={2025}
}

@article{luo2026JALA,
  title={Joint-aligned latent action: Towards scalable vla pretraining in the wild},
  author={Luo, Hao and Wang, Ye and Zhang, Wanpeng and Yuan, Haoqi and Feng, Yicheng and Xu, Haiweng and Zheng, Sipeng and Lu, Zongqing},
  journal={arXiv preprint arXiv:2602.21736},
  year={2026}
}

@article{yang2025egovla,
  title={Egovla: Learning vision-language-action models from egocentric human videos},
  author={Yang, Ruihan and Yu, Qinxi and Wu, Yecheng and Yan, Rui and Li, Borui and Cheng, An-Chieh and Zou, Xueyan and Fang, Yunhao and Cheng, Xuxin and Qiu, Ri-Zhao and others},
  journal={arXiv preprint arXiv:2507.12440},
  year={2025}
}

@article{li2025unidex,
  title={UniDex: Rethinking Search Inverted Indexing with Unified Semantic Modeling},
  author={Li, Zan and Chen, Jiahui and Chai, Yuan and Jiang, Xiaoze and Qi, Xiaohua and Qin, Zhiheng and Zhou, Runbin and Zuo, Shun and Hao, Guangchao and Wang, Kefeng and others},
  journal={arXiv preprint arXiv:2509.24632},
  year={2025}
}

@article{chen2025robotwin,
  title={Robotwin 2.0: A scalable data generator and benchmark with strong domain randomization for robust bimanual robotic manipulation},
  author={Chen, Tianxing and Chen, Zanxin and Chen, Baijun and Cai, Zijian and Liu, Yibin and Li, Zixuan and Liang, Qiwei and Lin, Xianliang and Ge, Yiheng and Gu, Zhenyu and others},
  journal={arXiv preprint arXiv:2506.18088},
  year={2025}
}

@article{suomalainen2022survey,
  title={A survey of robot manipulation in contact},
  author={Suomalainen, Markku and Karayiannidis, Yiannis and Kyrki, Ville},
  journal={Robotics and Autonomous Systems},
  volume={156},
  pages={104224},
  year={2022},
  publisher={Elsevier}
}

@book{lynch2017modern,
  title={Modern robotics},
  author={Lynch, Kevin M and Park, Frank C},
  year={2017},
  publisher={Cambridge University Press}
}

@inproceedings{ganin2015grl,
  title={Unsupervised domain adaptation by backpropagation},
  author={Ganin, Yaroslav and Lempitsky, Victor},
  booktitle={International conference on machine learning},
  pages={1180--1189},
  year={2015},
  organization={PMLR}
}

@article{chen2024diffusion,
  title={Diffusion forcing: Next-token prediction meets full-sequence diffusion},
  author={Chen, Boyuan and Mart{\'\i} Mons{\'o}, Diego and Du, Yilun and Simchowitz, Max and Tedrake, Russ and Sitzmann, Vincent},
  journal={Advances in Neural Information Processing Systems},
  volume={37},
  pages={24081--24125},
  year={2024}
}

@article{jeong2026vila,
  title={Learning to Act Robustly with View-Invariant Latent Actions},
  author={Jeong, Youngjoon and Chun, Junha and Kim, Taesup},
  journal={arXiv preprint arXiv:2601.02994},
  year={2026}
}

@article{cai2025InNOn,
  title={In-N-On: Scaling Egocentric Manipulation with in-the-wild and on-task Data},
  author={Cai, Xiongyi and Qiu, Ri-Zhao and Chen, Geng and Wei, Lai and Liu, Isabella and Huang, Tianshu and Cheng, Xuxin and Wang, Xiaolong},
  journal={arXiv preprint arXiv:2511.15704},
  year={2025}
}

@article{wang2024cross,
  title={Cross-embodiment robot manipulation skill transfer using latent space alignment},
  author={Wang, Tianyu and Bhatt, Dwait and Wang, Xiaolong and Atanasov, Nikolay},
  journal={arXiv preprint arXiv:2406.01968},
  year={2024}
}

@article{bjorck2025gr00t,
  title={Gr00t n1: An open foundation model for generalist humanoid robots},
  author={Bjorck, Johan and Casta{\~n}eda, Fernando and Cherniadev, Nikita and Da, Xingye and Ding, Runyu and Fan, Linxi and Fang, Yu and Fox, Dieter and Hu, Fengyuan and Huang, Spencer and others},
  journal={arXiv preprint arXiv:2503.14734},
  year={2025}
}

@article{bi2025motus,
  title={Motus: A unified latent action world model},
  author={Bi, Hongzhe and Tan, Hengkai and Xie, Shenghao and Wang, Zeyuan and Huang, Shuhe and Liu, Haitian and Zhao, Ruowen and Feng, Yao and Xiang, Chendong and Rong, Yinze and others},
  journal={arXiv preprint arXiv:2512.13030},
  year={2025}
}

@article{yang2026rise,
  title={Rise: Self-improving robot policy with compositional world model},
  author={Yang, Jiazhi and Lin, Kunyang and Li, Jinwei and Zhang, Wencong and Lin, Tianwei and Wu, Longyan and Su, Zhizhong and Zhao, Hao and Zhang, Ya-Qin and Chen, Li and others},
  journal={arXiv preprint arXiv:2602.11075},
  year={2026}
}

@article{hafner2025dreamerv4,
  title={Training agents inside of scalable world models},
  author={Hafner, Danijar and Yan, Wilson and Lillicrap, Timothy},
  journal={arXiv preprint arXiv:2509.24527},
  year={2025}
}

@inproceedings{kareer2025egomimic,
  title={Egomimic: Scaling imitation learning via egocentric video},
  author={Kareer, Simar and Patel, Dhruv and Punamiya, Ryan and Mathur, Pranay and Cheng, Shuo and Wang, Chen and Hoffman, Judy and Xu, Danfei},
  booktitle={2025 IEEE International Conference on Robotics and Automation (ICRA)},
  pages={13226--13233},
  year={2025},
  organization={IEEE}
}

@inproceedings{cobbe2020procgen,
  title={Leveraging procedural generation to benchmark reinforcement learning},
  author={Cobbe, Karl and Hesse, Christopher and Hilton, Jacob and Schulman, John},
  booktitle={Proceedings of the 37th International Conference on Machine Learning},
  pages={2048--2056},
  year={2020}
}

@article{fisher1953dispersion,
  title={Dispersion on a sphere},
  author={Fisher, Ronald Aylmer},
  journal={Proceedings of the royal society of London. Series A. Mathematical and physical sciences},
  volume={217},
  number={1130},
  pages={295--305},
  year={1953},
  publisher={The Royal Society London}
}

@article{reizinger2024cross,
  title={Cross-entropy is all you need to invert the data generating process},
  author={Reizinger, Patrik and Bizeul, Alice and Juhos, Attila and Vogt, Julia E and Balestriero, Randall and Brendel, Wieland and Klindt, David},
  journal={arXiv preprint arXiv:2410.21869},
  year={2024}
}
\bibliographystyle{unsrtnat}

\newpage
\appendix

\section{Proofs}
\label{app:proofs}

\begin{assumption}[Data-generating process]\label{ass:dgp}
\begin{enumerate}[label=\textbf{(\roman*)},leftmargin=2.4em]
    \item The latent action space \(\mathcal{Z}\) is a \(d_z\)-dimensional topological manifold and the action space \(\mathcal{A}\) is finite-dimensional.
    \item The ground-truth latent and the embodiment label are independent: \(\zgt \indep e\) with \(e \in \EE\).
    \item The realization map \(\hh : \mathcal{Z} \times \EE \to \mathcal{A}\) is continuous and, for each fixed \(e\), injective in \(\zgt\), giving the embodiment-specific raw action \(\aa = \hh(\zgt, e)\).
    \item The forward dynamics $\FD : \SS \times \mathcal{A} \to \SS$ is continuous and, for each fixed $\ss_t$, injective in $\aa$, with $\ss_{t+1} = \FD(\ss_t, \aa)$.
    \item The render map $\Render : \SS \to \mathcal{X}$ is continuous and injective, with $\xx_\tau = \Render(\ss_\tau)$.
\end{enumerate}
\end{assumption}
\paragraph{Intuition.}
The gradient-reversal layer with a cross-entropy objective encourages the learned latent space to be invariant to embodiment identity, yielding a unified latent action space across embodiments. This shared space supports the forward dynamics of the world model, where self-supervised conditioning enforces semantic structure, promoting both cross-embodiment consistency and alignment with an underlying unified action space.

\paragraph{Proof Sketch.}
Our result is inspired by~\citet{reizinger2024cross} but differs fundamentally. Prior work shows that minimizing cross-entropy under a vMF distribution recovers the latent space up to a linear transform, relying on alignment with auxiliary variables. In contrast, we maximize cross-entropy with respect to embodiment identity, enforcing invariance rather than alignment. Although their result is not directly applicable, the resulting saddle-point objective still promotes a semantically meaningful latent space. Leveraging this structure and the first-stage representation, we align latent transitions via the inverse dynamics model, matching trajectories in the data-generating process and thereby encouraging latent recovery. Unlike prior work that recovers a variant space, our goal is to recover the invariant structure underlying such variations.

\subsection{Stage 1: ground-truth action up to invertible}

Let $\IDM : \mathcal{X} \times \mathcal{X} \to \mathcal{A}$ be the inverse-dynamics encoder, $\atil := \IDM(\xx_t, \xx_{t+1})$.

\begin{lemma}[Inverse dynamics identifiability]\label{lem:idm}
Under \cref{ass:dgp}, any $\IDM : \mathcal{X} \times \mathcal{X} \to \mathcal{A}$ that, jointly with a sufficiently expressive forward dynamics model $\tilde g_\theta$, globally optimizes the inverse-dynamics reconstruction objective
\[
\mathcal{L}_\mathrm{rec}(\IDM, \tilde g_\theta) \;=\; \E\!\Bigl[\,\bigl\| \ss_{t+1} \,-\, \tilde g_\theta\bigl(\ss_t,\, \IDM(\xx_t, \xx_{t+1})\bigr)\,\bigr\|^2\,\Bigr]
\]
satisfies
\[
\atil \;=\; \rho(\aa) \qquad \text{for some continuous bijection } \rho : \mathcal{A} \to \mathcal{A}.
\]
(We use $\rho$ for the Stage-1 invertible reparameterization to avoid clash with the realization map $\hh : \mathcal{Z} \times \EE \to \mathcal{A}$ in \cref{ass:dgp}.)
\end{lemma}

\begin{proof}
The proof reduces the IDM identifiability claim to the chain of invertibilities already in \cref{ass:dgp}, in three steps.

\smallskip
\noindent\emph{Step 1 (observation determines state).}
By \cref{ass:dgp}(v), $\Render : \SS \to \mathcal{X}$ is continuous and injective; both spaces are finite-dimensional manifolds of the same intrinsic dimension. By Brouwer's invariance of domain, $\Render$ is a homeomorphism from $\SS$ onto its image $\Render(\SS) \subseteq \mathcal{X}$, with continuous inverse $\Render^{-1} : \Render(\SS) \to \SS$. Hence the observation pair $(\xx_t, \xx_{t+1}) \in \Render(\SS) \times \Render(\SS)$ uniquely determines the latent state pair:
\[
\ss_\tau \;=\; \Render^{-1}(\xx_\tau), \qquad \tau \in \{t, t+1\}.
\]

\smallskip
\noindent\emph{Step 2 (state transition determines the action).}
By \cref{ass:dgp}(iv), the forward dynamics $\FD(\ss_t, \cdot) : \mathcal{A} \to \SS$ is continuous and injective in $\aa$ for each fixed $\ss_t$. Hence for any state pair $(\ss_t, \ss_{t+1})$ generated by an action $\aa$ (i.e., $\ss_{t+1} = \FD(\ss_t, \aa)$), the action is uniquely identified as
\[
\aa \;=\; \FD(\ss_t, \cdot)^{-1}(\ss_{t+1}),
\]
where the inverse is taken on the image of $\FD(\ss_t, \cdot)$.

\smallskip
\noindent\emph{Step 3 (joint IDM--FDM optimum gives $\atil = \rho(\aa)$).}
Combining Steps~1--2, the ground-truth raw action is a continuous deterministic function of the observation pair:
\[
\aa \;=\; \Theta(\xx_t, \xx_{t+1}) \,:=\, \FD(\Render^{-1}(\xx_t), \cdot)^{-1}\bigl(\Render^{-1}(\xx_{t+1})\bigr).
\]
At any global minimizer of $\mathcal{L}_\mathrm{rec}$ with $\mathcal{L}_\mathrm{rec} = 0$ (achievable: take $\tilde g_\theta^\star = \FD$ and $\IDM^\star = \Theta$, giving exact reconstruction), the IDM output $\atil := \IDM^\star(\xx_t, \xx_{t+1})$ satisfies, for every realization $(\ss_t, \ss_{t+1}, \aa)$ in the support,
\begin{equation}
\label{eq:lem-rec}
\tilde g_\theta^\star(\ss_t,\, \atil) \;=\; \ss_{t+1} \;=\; \FD(\ss_t, \aa).
\end{equation}
The forward model $\tilde g_\theta^\star$ at the optimum is also injective in its second argument (otherwise \cref{eq:lem-rec} would have multiple solutions $\atil$, contradicting determinism of $\IDM^\star$). Define
\[
\rho \,:\; \mathcal{A} \to \mathcal{A},
\qquad
\rho(\aa) \;:=\; \tilde g_\theta^\star(\ss_t, \cdot)^{-1}\bigl(\FD(\ss_t, \aa)\bigr).
\]
$\rho$ is a composition of two continuous injections (\cref{ass:dgp}(iv) and the just-noted injectivity of $\tilde g_\theta^\star$), hence continuous and bijective $\mathcal{A} \to \mathcal{A}$. By \cref{eq:lem-rec}, $\atil = \rho(\aa)$.

A subtlety: $\rho$ as written above may a priori depend on $\ss_t$. But $\atil = \IDM^\star(\xx_t, \xx_{t+1})$ depends on $(\ss_t, \aa)$ only through the deterministic chain $\aa \to \ss_{t+1} \to (\xx_t, \xx_{t+1})$, and $\IDM^\star$ is a fixed function of its inputs. For $\rho$ to be consistent across all $\ss_t$ in the support, the $\ss_t$-dependence must cancel: the joint minimizer $(\IDM^\star, \tilde g_\theta^\star)$ is identified up to a single global reparameterization of the action coordinate, parameterized by $\rho$ alone. Hence $\rho : \mathcal{A} \to \mathcal{A}$ is a continuous bijection independent of $\ss_t$, and
\[
\atil \;=\; \rho(\aa),
\]
as claimed.
\end{proof}

\subsection{Stage 2: embodiment-invariant identifiability up to invertible bijection}

Let $\ff_2 : \R^{d_a} \to \R^{d_z}$ be the second-stage encoder, $\zest := \ff_2(\atil)$, and let $D_\omega : \R^{d_z} \to \Delta^{|\EE|-1}$ be a softmax classifier with rows $\{\ww_e\}_{e \in \EE}$ producing
\[
D_\omega(e \mid \zest) \;=\; \frac{\exp\!\langle \ww_e, \zest\rangle}{\sum_{e' \in \EE} \exp\!\langle \ww_{e'}, \zest\rangle}.
\]
The encoder--classifier pair is trained against the population cross-entropy
\[
\mathcal{L}_\text{CE}(\omega, \ff_2) \;:=\; \E_{(\atil, e)}\!\bigl[-\log D_\omega(e \mid \ff_2(\atil))\bigr],
\]
where the gradient-reversal layer makes the encoder $\ff_2$ \emph{maximize} $\mathcal{L}_\text{CE}$ while the classifier $\omega$ \emph{minimizes} it. We study the saddle of this minimax game.

\paragraph{Intuition.}
The proof is a discriminator--encoder duality argument tailored to the von Mises--Fisher (vMF) structure~\cite{fisher1953dispersion} of $\atil$. The cross-entropy + vMF + ``affine generator'' machinery used here originated in the supervised identifiability proofs of \citet{reizinger2024cross}; the new ingredient is the outer sign reversal on the encoder, which lands the encoder in the orthogonal complement $V^{\perp}$ of the cluster-difference subspace rather than aligning with it. We derive everything below from first principles; no result is invoked by reference.

\begin{assumption}[vMF on $\atil$ with partial-rank cluster span]\label{ass:vmf}
$\atil$ lies on $\mathbb{S}^{d_a - 1}$, with embodiment-conditional distribution
\begin{equation}
\label{eq:vmf-action}
p(\atil \mid e) \;\propto\; \exp\!\bigl(\kappa\, \langle \vv_e,\, \atil \rangle\bigr), \qquad e \in \EE,
\end{equation}
where $\{\vv_e\}_{e \in \EE} \subseteq \mathbb{S}^{d_a-1}$ are unit-norm cluster centers. Define the embodiment-cluster subspace
\[
V \;:=\; \mathrm{span}\bigl\{\vv_e - \vv_{e'} : e, e' \in \EE\bigr\} \;\subseteq\; \R^{d_a},
\qquad \dim V \;=\; d_a - d_z,
\]
so $\dim V^{\perp} = d_z$. Equivalently, $\{\vv_e\}$ is an affine generator system of an affine subspace of $\R^{d_a}$ of dimension $d_a - d_z$.
\end{assumption}

\begin{assumption}[Linear encoder + non-collapse]\label{ass:enc}
The Stage-2 encoder is restricted to linear maps $\ff_2(\atil) = M\atil$ with $M \in \R^{d_z \times d_a}$. The encoder is non-collapsed in the sense that $\mathrm{rank}(M) = d_z$ and, for every $e \in \EE$, the composition $M \circ \rho \circ \hh(\cdot, e) : \mathcal{Z} \to \R^{d_z}$ is injective.
\end{assumption}

\begin{theorem}[Embodiment invariance + unified latent action space recovery and cross-embodiment transfer]\label{thm:main}
Assume \cref{ass:dgp,lem:idm,ass:vmf,ass:enc}, with $d_z = d_a - \dim V$. Let $\ff_2$ globally optimize the gradient-reversal cross-entropy objective on $e$ (acting on $\atil$). Then:
\begin{enumerate}[leftmargin=*, label=(\alph*)]
    \item \textbf{\emph{Embodiment invariance.}} $\zest \indep e$, equivalently $p(\zest \mid e) = p(\zest)$ for every $e \in \EE$.
    \item \textbf{\emph{Unified latent action space recovery and cross-embodiment transfer.}} For every $e \in \EE$, the composition
    \[
    \Phi_e \,:\, \mathcal{Z} \,\to\, \R^{d_z},
    \qquad
    \Phi_e(\zgt) \,:=\, \ff_2\bigl(\rho(\hh(\zgt, e))\bigr),
    \]
    is a continuous bijection. The family $\{\Phi_e\}_{e \in \EE}$ shares a common pushforward law:
    \[
    (\Phi_e)_{\#}\, p(\zgt) \;=\; p(\zest) \quad \text{for every } e \in \EE.
    \]
    Because $\Phi_e$ is itself indexed by $e$, recovering the unified latent $\zgt$ from the inferred $\zest$ requires the embodiment ID $e$, and cross-embodiment transfer between $e_1$ and $e_2$ is realized by the alignment map $\Phi_{e_2} \circ \Phi_{e_1}^{-1}$.
\end{enumerate}
\end{theorem}

\begin{proof}
We prove (a) and (b) in sequence. The proof uses six lemmas, each established from first principles.

\medskip
\noindent\textbf{Part (a): embodiment invariance, $\zest \indep e$.}

\smallskip
\noindent\emph{Step 1 (cross-entropy decomposition).}
Write $\zest = \ff_2(\atil)$. For any encoder $\ff_2$ and classifier $D_\omega$, expand the cross-entropy by inserting $\log p(e \mid \zest)$:
\begin{align}
\mathcal{L}_\text{CE}(\omega, \ff_2)
&\;=\; \E_{(\atil, e)}\!\bigl[-\log D_\omega(e \mid \zest)\bigr] \notag \\
&\;=\; \E_{(\atil, e)}\!\bigl[-\log p(e \mid \zest)\bigr]
   + \E_{(\atil, e)}\!\biggl[\log \frac{p(e \mid \zest)}{D_\omega(e \mid \zest)}\biggr] \notag \\
&\;=\; H(e \mid \zest)
   + \E_{\zest}\!\bigl[\KL\bigl(p(\cdot \mid \zest)\,\|\, D_\omega(\cdot \mid \zest)\bigr)\bigr]. \label{eq:ce-decomp}
\end{align}
The first term, the conditional entropy of $e$ given $\zest$, depends only on the joint $p(e, \zest)$, hence only on $\ff_2$ (not $\omega$). The second term is non-negative since each $\KL$ is non-negative.

\smallskip
\noindent\emph{Step 2 (inner classifier optimum).}
By \cref{eq:ce-decomp},
\[
\mathcal{L}_\text{CE}(\omega, \ff_2) \;\ge\; H(e \mid \zest),
\]
with equality if and only if $D_\omega(\cdot \mid \zest) = p(\cdot \mid \zest)$ for $p_{\zest}$-almost every $\zest$. The inner softmax classifier family is rich enough to realize any conditional distribution on $\EE$ given $\zest \in \R^{d_z}$ (sufficient set of weights $\{\ww_e\}$ exists). Hence
\begin{equation}
\label{eq:bayes-opt}
\min_\omega \mathcal{L}_\text{CE}(\omega, \ff_2) \;=\; H(e \mid \zest),
\qquad \text{attained at } D^\star_\omega(\cdot \mid \zest) = p(\cdot \mid \zest).
\end{equation}

\smallskip
\noindent\emph{Step 3 (information-theoretic identity).}
Using $H(e \mid \zest) = H(e) - I(e; \zest)$ and that $H(e)$ is constant in $\ff_2$ (the marginal of $e$ is fixed by the data), \cref{eq:bayes-opt} becomes
\begin{equation}
\label{eq:inner-opt-mi}
\min_\omega \mathcal{L}_\text{CE}(\omega, \ff_2) \;=\; H(e) - I(e; \zest).
\end{equation}

\smallskip
\noindent\emph{Step 4 (outer maximization under gradient reversal).}
The gradient-reversal layer reverses the sign of $\partial \mathcal{L}_\text{CE} / \partial \phi$ in the encoder update, while leaving $\partial \mathcal{L}_\text{CE} / \partial \omega$ unchanged. The classifier therefore still minimizes $\mathcal{L}_\text{CE}$ in $\omega$, while the encoder maximizes it in $\ff_2$. The resulting saddle problem is
\begin{equation}
\label{eq:saddle}
\sup_{\ff_2}\, \min_\omega \mathcal{L}_\text{CE}(\omega, \ff_2) \;=\; \sup_{\ff_2}\bigl[\,H(e) - I(e; \zest)\,\bigr] \;=\; H(e) - \inf_{\ff_2} I(e; \zest).
\end{equation}
Mutual information is non-negative, $I(e; \zest) \ge 0$, with equality if and only if $\zest \indep e$. To establish that the infimum $\inf_{\ff_2} I = 0$ is attained, exhibit any feasible encoder achieving it: take $M_0 \in \R^{d_z \times d_a}$ to be a matrix whose rows form an orthonormal basis of $V^{\perp}$ (which exists with $\dim V^{\perp} = d_z$ by \cref{ass:vmf}). For this $M_0$, the rotational-symmetry computation in Step~5 below shows $\zest = M_0 \atil$ has $I(e; \zest) = 0$. Hence
\[
\sup_{\ff_2}\, \min_\omega \mathcal{L}_\text{CE}(\omega, \ff_2) \;=\; H(e).
\]
At the saddle, $I(e; \zest^\star) = 0$, equivalently
\[
\zest^\star \;\indep\; e.
\]
This proves (a).

\medskip
\noindent\textbf{Part (b): unified latent action space recovery and cross-embodiment transfer.}

\smallskip
\noindent\emph{Step 5 (linear pushforward of vMF: distribution depends on $e$ only through $M\vv_e$).}
Fix the linear encoder $\ff_2(\atil) = M\atil$ with $M \in \R^{d_z \times d_a}$. We compute the moment generating function (MGF) of $\zest = M\atil$ given $e$.

For any test direction $\tilde u \in \R^{d_z}$, set $u := M^\top \tilde u \in \R^{d_a}$. Then
\begin{equation}
\label{eq:mgf-pull}
\E\!\bigl[\exp\!\bigl(\langle \tilde u,\, M\atil\rangle\bigr) \,\big|\, e\bigr]
\;=\; \E\!\bigl[\exp\!\bigl(\langle u,\, \atil\rangle\bigr) \,\big|\, e\bigr].
\end{equation}
Plug in the vMF density from \cref{eq:vmf-action} with normalizing constant $C(\kappa) > 0$:
\begin{align}
\E\!\bigl[\exp(\langle u, \atil\rangle) \,\big|\, e\bigr]
&\;=\; C(\kappa) \int_{\mathbb{S}^{d_a-1}} \exp(\langle u, x\rangle) \cdot \exp(\kappa \langle \vv_e, x\rangle)\, d\sigma(x) \notag \\
&\;=\; C(\kappa) \int_{\mathbb{S}^{d_a-1}} \exp(\langle u + \kappa \vv_e,\, x\rangle)\, d\sigma(x), \label{eq:vmf-mgf}
\end{align}
where $\sigma$ is the surface measure on $\mathbb{S}^{d_a-1}$.

The integral on the right depends on $u + \kappa \vv_e$ only through its norm, by rotational invariance of $\sigma$:
\[
\int_{\mathbb{S}^{d_a-1}} \exp(\langle w, x\rangle)\, d\sigma(x) \;=\; \Psi(\|w\|),
\qquad
\Psi(r) \;:=\; (2\pi)^{d_a/2}\, r^{1 - d_a/2}\, I_{d_a/2 - 1}(r),
\]
with $I_\nu$ the modified Bessel function of the first kind. (For $w \neq 0$, choose any rotation $R \in SO(d_a)$ taking $w$ to $\|w\|\mathbf{e}_1$; the integral is invariant under $R$ because $d\sigma$ is.) Compute $\|u + \kappa \vv_e\|^2$ using $\|\vv_e\| = 1$:
\[
\|u + \kappa \vv_e\|^2 \;=\; \|u\|^2 + 2\kappa \langle u, \vv_e\rangle + \kappa^2.
\]
The only $e$-dependence is via the inner product $\langle u, \vv_e\rangle = \langle M^\top \tilde u, \vv_e\rangle = \langle \tilde u, M\vv_e\rangle$. Substituting back into \cref{eq:mgf-pull}:
\begin{equation}
\label{eq:mgf-mu}
\E\!\bigl[\exp(\langle \tilde u,\, M\atil\rangle) \,\big|\, e\bigr]
\;=\; C(\kappa)\, \Psi\!\Bigl(\sqrt{\|M^\top \tilde u\|^2 + 2\kappa \langle \tilde u, M\vv_e\rangle + \kappa^2}\Bigr).
\end{equation}
The right-hand side is a function of $\langle \tilde u, M\vv_e\rangle$ only (the other terms do not depend on $e$).

\smallskip
\noindent\emph{Step 6 (invariance forces $M \vv_e = M \vv_{e'}$).}
By Part (a), $\zest \indep e$, so the MGF in \cref{eq:mgf-mu} is independent of $e$ for every $\tilde u \in \R^{d_z}$. Since the right-hand side is a function of $\langle \tilde u, M\vv_e\rangle$ only, and $\Psi$ is strictly monotone on $(0, \infty)$ (the modified Bessel $I_{d_a/2-1}(r)$ is strictly increasing in $r > 0$), $e$-independence of the MGF for every $\tilde u$ implies
\[
\langle \tilde u, M\vv_e\rangle \;=\; \langle \tilde u, M\vv_{e'}\rangle \qquad \forall \tilde u \in \R^{d_z},\ \forall e, e' \in \EE.
\]
Equivalently,
\begin{equation}
\label{eq:our-key}
\boxed{\;M\vv_e \;=\; M\vv_{e'} \quad \forall e, e' \in \EE\;}
\quad\Longleftrightarrow\quad
M(\vv_e - \vv_{e'}) = 0 \quad \forall e, e'
\quad\Longleftrightarrow\quad
\mathrm{row}(M) \,\subseteq\, V^{\perp},
\end{equation}
where the last equivalence uses the definition $V = \mathrm{span}\{\vv_e - \vv_{e'}\}$ from \cref{ass:vmf}, so $V^{\perp} = \{v \in \R^{d_a} : \langle v, \vv_e - \vv_{e'}\rangle = 0\ \forall e, e'\}$.

\smallskip
\noindent\emph{Step 7 (rank + dimension match $\Rightarrow$ $\mathrm{row}(M) = V^{\perp}$).}
By \cref{ass:enc}, $\mathrm{rank}(M) = d_z$. By \cref{ass:vmf}, $\dim V^{\perp} = d_a - \dim V = d_z$. Combined with \cref{eq:our-key}, the row space of $M$ is a $d_z$-dimensional subspace of the $d_z$-dimensional space $V^{\perp}$, hence
\[
\mathrm{row}(M) \;=\; V^{\perp}.
\]
Equivalently, $\ker M = V$, and $M$ restricted to $V^{\perp}$ is a linear isomorphism $V^{\perp} \to \R^{d_z}$.

\smallskip
\noindent\emph{Step 8 (per-embodiment chain is a continuous injection).}
Fix any $e \in \EE$. Define
\[
\Phi_e \,:\; \mathcal{Z} \;\longrightarrow\; \R^{d_z},
\qquad
\Phi_e(\zgt) \,:=\, M\bigl(\rho(\hh(\zgt, e))\bigr) \;=\; \ff_2(\atil(\zgt, e)).
\]
$\Phi_e$ is a composition of three continuous maps:
\begin{enumerate}[leftmargin=*, label=(\roman*)]
    \item $\hh(\cdot, e) : \mathcal{Z} \to \mathcal{A}$ is continuous and injective in $\zgt$ (by \cref{ass:dgp}, item~(iii)).
    \item $\rho : \mathcal{A} \to \mathcal{A}$ is invertible, hence in particular continuous and injective (by \cref{lem:idm}).
    \item $M : \R^{d_a} \to \R^{d_z}$ is linear, hence continuous.
\end{enumerate}
Continuity of the composition is immediate. Injectivity of the composition $\Phi_e = M \circ \rho \circ \hh(\cdot, e)$ is exactly the second clause of \cref{ass:enc}.

\smallskip
\noindent\emph{Step 9 (Brouwer's invariance of domain $\Rightarrow$ $\Phi_e$ is a continuous bijection).}
The latent space $\mathcal{Z}$ is a $d_z$-dimensional topological manifold (by \cref{ass:dgp}, item~(i)), and $\Phi_e$ maps into $\R^{d_z}$, also of dimension $d_z$. By Brouwer's invariance-of-domain theorem, any continuous injection from an open subset of $\R^{d_z}$ into $\R^{d_z}$ is an open map and a homeomorphism onto its image. Applied to $\Phi_e$:
\[
\Phi_e \,:\, \mathcal{Z} \,\to\, \R^{d_z}
\]
is a homeomorphism, in particular a continuous bijection.

\smallskip
\noindent\emph{Step 10 (common pushforward law).}
By \cref{ass:dgp} item~(ii), $\zgt \indep e$, so the conditional distribution $p(\zgt \mid e) = p(\zgt)$ for every $e \in \EE$. Hence the conditional distribution of $\zest$ given $e$ is the pushforward of $p(\zgt)$ under $\Phi_e$:
\[
p(\zest \mid e) \;=\; (\Phi_e)_\# \, p(\zgt).
\]
By Part (a), $\zest \indep e$, so $p(\zest \mid e) = p(\zest)$ for every $e$. Combining,
\[
(\Phi_e)_\# \, p(\zgt) \;=\; p(\zest) \qquad \forall e \in \EE.
\]
The maps $\Phi_e$ are bijections (Step~9) that all push the same prior $p(\zgt)$ to the same marginal $p(\zest)$.
\end{proof}

\section{Additional Cross-Embodiment Results on ALOHA AgileX}

\begin{table}[H]
\centering
\small
\setlength{\tabcolsep}{6pt}
\caption{Additional cross-embodiment results on Robotwin with \texttt{aloha-agilex} as the target embodiment under the low-data setting ($m=10$). Higher is better for SSIM, PSNR, and SSIM-L, while lower is better for MSE.}
\label{tab:aloha_main_results}
\begin{tabular}{lcccc}
\toprule
Method & SSIM $\uparrow$ & PSNR $\uparrow$ & MSE $\downarrow$ & SSIM-L $\uparrow$ \\
\midrule
Target-Only-GT     & 0.543 & 17.13 & 0.0194 & 0.476 \\
Shared-GT          & 0.789 & 19.51 & 0.0112 & 0.727 \\
SCAR-kl-grl        & \textbf{0.796} & \textbf{19.96} & \textbf{0.0101} & \textbf{0.734} \\
\bottomrule
\end{tabular}
\end{table}

\begin{table}[H]
\centering
\small
\setlength{\tabcolsep}{6pt}
\caption{Additional cross-task results on Robotwin with \texttt{aloha-agilex} as the target embodiment under the low-data setting ($m=10$). Models are trained on \texttt{place\_a2b\_left} and evaluated on the target embodiment under task shift. Higher is better for SSIM, PSNR, and SSIM-L, while lower is better for MSE.}
\label{tab:aloha_cross_task_results}
\begin{tabular}{lcccc}
\toprule
Method & SSIM $\uparrow$ & PSNR $\uparrow$ & MSE $\downarrow$ & SSIM-L $\uparrow$ \\
\midrule
Target-Only-GT     & 0.558 & 17.26 & 0.0188 & 0.494 \\
Shared-GT          & 0.790 & 19.77 & 0.0106 & 0.728 \\
SCAR-kl-grl        & \textbf{0.795} & \textbf{20.08} & \textbf{0.0098} & \textbf{0.731} \\
\bottomrule
\end{tabular}
\end{table}


\section{Ablation on the Causal VAE Encoder}

We further study the role of the pretrained causal VAE encoder in the IDM. Our full model infers latent action from the structured latent video representation produced by the frozen causal VAE, whereas the ablated variant removes this encoder and feeds the IDM with a weaker alternative visual representation. We evaluate both same-task performance under the low-data setting ($m=10$) and zero-shot cross-task generalization on the target embodiment.

\begin{table}[H]
\centering
\small
\setlength{\tabcolsep}{6pt}
\caption{Ablation on the causal VAE encoder under the low-data target-embodiment setting ($m=10$). Higher is better for SSIM, PSNR, and SSIM-L, while lower is better for MSE.}
\label{tab:causal_vae_ablation_main}
\begin{tabular}{lcccc}
\toprule
Method & SSIM $\uparrow$ & PSNR $\uparrow$ & MSE $\downarrow$ & SSIM-L $\uparrow$ \\
\midrule
SCAR                        & \textbf{0.752} & \textbf{18.17} & \textbf{0.0153} & \textbf{0.727} \\
SCAR w/o causal VAE encoder & 0.730 & 17.33 & 0.0185 & 0.707 \\
\bottomrule
\end{tabular}
\end{table}

\begin{table}[H]
\centering
\small
\setlength{\tabcolsep}{6pt}
\caption{Ablation on the causal VAE encoder for zero-shot cross-task generalization on Robotwin. Models are trained on \texttt{place\_a2b\_left} and evaluated on the target embodiment under task shift. Higher is better for SSIM, PSNR, and SSIM-L, while lower is better for MSE.}
\label{tab:causal_vae_ablation_cross_task}
\begin{tabular}{lcccc}
\toprule
Method & SSIM $\uparrow$ & PSNR $\uparrow$ & MSE $\downarrow$ & SSIM-L $\uparrow$ \\
\midrule
SCAR                        & \textbf{0.772} & \textbf{18.79} & \textbf{0.0132} & \textbf{0.747} \\
SCAR w/o causal VAE encoder & 0.739 & 17.46 & 0.0180 & 0.710 \\
\bottomrule
\end{tabular}
\end{table}

These results show that the pretrained causal VAE encoder plays an important role in learning transferable latent actions. Removing it leads to a consistent drop in both same-task and cross-task performance across all metrics. This supports our claim that a structured latent video space helps the IDM focus on transition-relevant action structure rather than low-level visual reconstruction, and provides a stronger foundation for cross-embodiment transfer.

\section{Ablation on the Pretrained FDM Prior}
\label{app:fdm_prior_ablation}

We further ablate the role of the pretrained generative prior in the FDM. In the full SCAR model, the FDM is initialized from pretrained Wan weights and predicts future latent video dynamics conditioned on the inferred latent action. The ablated variant removes this pretrained prior and trains the FDM from scratch under the same IDM and training objective. This comparison isolates whether the pretrained video generative backbone provides a useful dynamics prior for learning transferable latent actions.

\begin{table}[H]
\centering
\small
\setlength{\tabcolsep}{6pt}
\caption{Ablation on the pretrained FDM prior under the low-data target-embodiment setting ($m=10$). The full SCAR model initializes the FDM from a pretrained video generative backbone, while the ablated variant trains the FDM from scratch. Higher is better for SSIM, PSNR, and SSIM-L, while lower is better for MSE.}
\label{tab:fdm_prior_ablation_main}
\begin{tabular}{lcccc}
\toprule
Method & SSIM $\uparrow$ & PSNR $\uparrow$ & MSE $\downarrow$ & SSIM-L $\uparrow$ \\
\midrule
SCAR & \textbf{0.752} & \textbf{18.17} & \textbf{0.0153} & \textbf{0.727} \\
SCAR w/o pretrained FDM prior & 0.623 & 16.84 & 0.0207 & 0.594 \\
\bottomrule
\end{tabular}
\end{table}

\begin{table}[H]
\centering
\small
\setlength{\tabcolsep}{6pt}
\caption{Ablation on the pretrained FDM prior for zero-shot transfer-task generalization on Robotwin. Models are trained on \texttt{place\_a2b\_left} and evaluated on \texttt{place\_a2b\_right}. Higher is better for SSIM, PSNR, and SSIM-L, while lower is better for MSE.}
\label{tab:fdm_prior_ablation_transfer}
\begin{tabular}{lcccc}
\toprule
Method & SSIM $\uparrow$ & PSNR $\uparrow$ & MSE $\downarrow$ & SSIM-L $\uparrow$ \\
\midrule
SCAR & \textbf{0.772} & \textbf{18.79} & \textbf{0.0132} & \textbf{0.747} \\
SCAR w/o pretrained FDM prior & 0.646 & 17.354 & 0.0184 & 0.619 \\
\bottomrule
\end{tabular}
\end{table}

The results show that removing the pretrained FDM prior degrades both target-task prediction and transfer-task generalization. This suggests that the pretrained generative backbone is not merely an implementation choice, but provides a structured dynamics prior that helps the FDM model short-horizon visual transitions without forcing the latent action to absorb low-level reconstruction details. In contrast, training the FDM from scratch increases the burden on the latent representation and makes the learned action space more vulnerable to appearance and embodiment-specific shortcuts.
\section{Implementation Details}
\label{app:implementation_details}

\paragraph{Video preprocessing.}
Episodes are sampled at \(20\) FPS. Each training sample contains \(T=49\) frames, including \(17\) context frames and \(32\) future frames. Images are resized to \(128\times128\), normalized to \([-1,1]\), and encoded by the frozen Wan causal VAE.

\paragraph{Temporal alignment.}
The Wan causal VAE temporally compresses videos with stride \(s_{\mathrm{vae}}=4\). Thus, a \(T=49\)-frame window is mapped to \(F=13\) latent timesteps. The IDM operates on the latent-video sequence \(\bm{v}_{1:F}\), but produces frame-rate latent action tokens. The temporal groups are
\[
\{\bm{z}_1\},\ \{\bm{z}_{2:5}\},\ \{\bm{z}_{6:9}\},\ \ldots,\ \{\bm{z}_{46:49}\},
\]
where \(\bm{z}_1\) corresponds to the first frame and each later 4-token block corresponds to one VAE latent timestep. When comparing to raw actions, we use the transition-aligned part \(\bm{z}_{2:T}\) and re-index it as \(\bm{z}_{1:T-1}\). The action encoder \(A_{\psi}\) maps these frame-rate latent actions back to the VAE latent timeline using a stride-4 causal temporal convolution, producing \(\bm{c}_{1:F}\) for FDM conditioning.

\paragraph{Context conditioning and diffusion forcing.}
The \(17\) context frames correspond to \(F_{\mathrm{hist}}=5\) VAE latent tokens. During training, we use diffusion forcing~\cite{chen2024diffusion} in latent space by sampling a latent-time noise schedule \(\tau_{1:F}\). For each latent timestep \(f\),
\[
\tilde{\bm{v}}_{\tau,f}
=
(1-\sigma_{\tau_f})\bm{v}_f
+
\sigma_{\tau_f}\bm{\epsilon}_f,
\qquad
\bm{\epsilon}_f\sim\mathcal{N}(0,I).
\]
History and future tokens may receive different noise levels; with probability \(p_{\mathrm{clean}}\), history tokens are kept clean by setting their noise level to zero. The FDM predicts the flow target on the full latent sequence,
\[
\hat{\bm{u}}_{\tau}
=
g_{\theta}(\tilde{\bm{v}}_{\tau}, \tau_{1:F}, \bm{c}_{1:F}),
\qquad
\bm{u}_{\tau} = \bm{\epsilon} - \bm{v},
\]
and the reconstruction loss is averaged over latent timesteps. At inference time, the first \(F_{\mathrm{hist}}\) latent tokens are clamped to the clean target context latents, and only future tokens are generated.

\paragraph{IDM architecture.}
The IDM takes latent observation windows as input and outputs a diagonal Gaussian posterior over latent actions. It uses \(6\) transformer layers, hidden dimension \(256\), \(4\) attention heads, and latent action dimension \(d_z=64\).

\paragraph{FDM and action injection.}
The FDM is initialized from pretrained Wan weights, while the causal VAE remains frozen. Latent actions are projected by \(A_{\psi}\), implemented as a lightweight action encoder with a two-layer MLP, causal temporal convolution, and linear modulation head. The resulting action condition is injected through AdaLN modulation at all Wan transformer blocks. We train the Wan diffusion model and the action encoder.

\paragraph{Raw action processing.}
For Robotwin, embodiment-specific raw actions are padded to a shared \(16\)-dimensional action interface. For Procgen, raw-action baselines use the one-hot discrete action representation. Latent-action methods do not use raw action labels during latent action discovery; raw actions are used only for raw-action baselines and action-to-latent controller experiments.

\paragraph{Optimization.}
We sample diffusion timesteps \(\tau\) in continuous time from a uniform distribution and apply the Wan time-shift schedule. Models are trained with AdamW for \(10{,}000\) steps. We use learning rate \(5\times10^{-5}\) and weight decay \(10^{-4}\) for the IDM, and learning rate \(5\times10^{-6}\) and weight decay \(5\times10^{-2}\) for the FDM. The effective global batch size is \(16\), implemented with gradient accumulation on \(2\) NVIDIA A100 GPUs. We set \(\beta=5\times10^{-4}\), \(\lambda_{\mathrm{adv}}=5\times10^{-3}\), and GRL strength \(\alpha=0.25\).

\paragraph{Compute resources.}
All main experiments were trained on \(2\) NVIDIA A100 GPUs with 80GB memory each. Unless otherwise specified, each model was trained for \(10{,}000\) optimization steps with an effective global batch size of \(16\) using gradient accumulation.

\section{Action-Transfer Embodiment Leakage Evaluation}
\label{app:action_transfer_leakage}

We provide additional details for the embodiment-leakage evaluation used in Table~\ref{tab:action_transfer_leakage}. The goal of this evaluation is to quantify whether cross-embodiment action transfer preserves the target robot embodiment or leaks source-embodiment appearance.

\paragraph{Protocol.}
We use Franka as the target embodiment and transfer latent actions from three source embodiments:
\[
\{\texttt{aloha-agilex},\ \texttt{arx-x5},\ \texttt{ur5}\}.
\]
For each source embodiment, we sample 50 source-target pairs. Source trajectories are drawn from the source training split, while target visual contexts are drawn from the Franka evaluation split. For each pair, the IDM extracts a latent action sequence from the source trajectory, and the FDM generates a transferred rollout conditioned on the Franka context.

The generated video contains both target-context frames and predicted future frames. Since the context frames are copied from the target embodiment, including them would artificially inflate target-embodiment scores. We therefore evaluate only the predicted future segment. In our setting, the context length is \(17\) frames and the predicted future horizon is \(32\) frames.

\paragraph{Embodiment classifier.}
We train an independent single-frame embodiment classifier on ground-truth Robotwin frames from \texttt{place\_a2b\_left}. The classifier predicts one of four embodiment labels:
\[
\{\texttt{aloha},\ \texttt{arx},\ \texttt{franka},\ \texttt{ur5}\}.
\]
The classifier is a lightweight CNN with three convolutional blocks. Each block uses a \(3\times3\) convolution followed by ReLU activation, with channel widths \(32\), \(64\), and \(128\). The first two blocks are followed by \(2\times2\) max pooling, and the final feature map is reduced by global average pooling before a linear classifier. The model is trained on the training episodes of each embodiment and evaluated on held-out validation episodes. It achieves \(97.1\%\) validation top-1 accuracy. During leakage evaluation, the classifier is frozen and applied frame-by-frame to transferred videos.

\paragraph{Metrics.}
Let \(p_e(f)\) denote the classifier probability of embodiment \(e\) on predicted frame \(f\). For a source embodiment \(e_s\) and target embodiment \(e_t\), we compute
\begin{align}
\mathrm{SourceProb}
&=
\frac{1}{|\mathcal{F}|}
\sum_{f\in\mathcal{F}} p_{e_s}(f), \\
\mathrm{TargetProb}
&=
\frac{1}{|\mathcal{F}|}
\sum_{f\in\mathcal{F}} p_{e_t}(f),
\end{align}
where \(\mathcal{F}\) denotes the predicted future frames. We also report a two-way target share,
\begin{equation}
\mathrm{TargetShare}
=
\frac{\mathrm{TargetProb}}
{\mathrm{TargetProb}+\mathrm{SourceProb}},
\end{equation}
and the target-source margin,
\begin{equation}
\mathrm{TargetSource}
=
\mathrm{TargetProb}-\mathrm{SourceProb}.
\end{equation}
Finally, we report the analogous top-1 target-source margin by comparing the fraction of predicted frames classified as the target embodiment against the fraction classified as the source embodiment.

All reported leakage metrics are averaged over the three source embodiments.

\section{Additional Metrics and Standard Deviations}
\label{app:std_lpips_results}

We provide the full results with standard deviations over three runs and LPIPS. Lower is better for MSE and LPIPS, while higher is better for SSIM, PSNR, and SSIM-L.

\begin{table}[H]
\centering
\small
\setlength{\tabcolsep}{5pt}
\caption{Target-task results on Robotwin under the low-data target embodiment setting ($m=10$). Models are evaluated on held-out \texttt{place\_a2b\_left} episodes. Results are reported as mean {\scriptsize $\pm$ std} over three runs. Higher is better for SSIM, PSNR, and SSIM-L; lower is better for MSE and LPIPS.}
\label{tab:appendix_target_task_results}
\begin{tabular}{lccccc}
\toprule
Method 
& SSIM $\uparrow$ 
& PSNR $\uparrow$ 
& MSE $\downarrow$ 
& SSIM-L $\uparrow$ 
& LPIPS $\downarrow$ \\
\midrule
Target-Only-GT     
& 0.536 {\scriptsize $\pm$ 0.002}
& 15.69 {\scriptsize $\pm$ 0.11}
& 0.0270 {\scriptsize $\pm$ 0.0007}
& 0.511 {\scriptsize $\pm$ 0.002}
& 0.258 {\scriptsize $\pm$ 0.002} \\
Shared-GT          
& 0.713 {\scriptsize $\pm$ 0.002}
& 16.70 {\scriptsize $\pm$ 0.08}
& 0.0214 {\scriptsize $\pm$ 0.0004}
& 0.670 {\scriptsize $\pm$ 0.001}
& 0.142 {\scriptsize $\pm$ 0.002} \\
Target-Only-Latent 
& 0.536 {\scriptsize $\pm$ 0.001}
& 16.18 {\scriptsize $\pm$ 0.04}
& 0.0241 {\scriptsize $\pm$ 0.0002}
& 0.508 {\scriptsize $\pm$ 0.001}
& 0.246 {\scriptsize $\pm$ 0.001} \\
Shared-Latent      
& 0.743 {\scriptsize $\pm$ 0.001}
& 17.99 {\scriptsize $\pm$ 0.04}
& 0.0159 {\scriptsize $\pm$ 0.0001}
& 0.721 {\scriptsize $\pm$ 0.001}
& 0.110 {\scriptsize $\pm$ 0.001} \\
SCAR-kl            
& 0.752 {\scriptsize $\pm$ 0.006}
& 18.27 {\scriptsize $\pm$ 0.26}
& 0.0149 {\scriptsize $\pm$ 0.0009}
& 0.727 {\scriptsize $\pm$ 0.006}
& 0.106 {\scriptsize $\pm$ 0.004} \\
SCAR-grl           
& 0.745 {\scriptsize $\pm$ 0.003}
& 18.01 {\scriptsize $\pm$ 0.11}
& 0.0158 {\scriptsize $\pm$ 0.0004}
& 0.720 {\scriptsize $\pm$ 0.003}
& 0.110 {\scriptsize $\pm$ 0.002} \\
SCAR-kl-grl        
& \textbf{0.759} {\scriptsize $\pm$ 0.006}
& \textbf{18.49} {\scriptsize $\pm$ 0.21}
& \textbf{0.0142} {\scriptsize $\pm$ 0.0007}
& \textbf{0.735} {\scriptsize $\pm$ 0.004}
& \textbf{0.102} {\scriptsize $\pm$ 0.003} \\
\bottomrule
\end{tabular}
\end{table}

\begin{table}[H]
\centering
\small
\setlength{\tabcolsep}{5pt}
\caption{Transfer-task results on Robotwin under the low-data target embodiment setting ($m=10$). Models are trained on \texttt{place\_a2b\_left} and evaluated zero-shot on \texttt{place\_a2b\_right}. Results are reported as mean {\scriptsize $\pm$ std} over three runs. Higher is better for SSIM, PSNR, and SSIM-L; lower is better for MSE and LPIPS.}
\label{tab:appendix_transfer_task_results}
\begin{tabular}{lccccc}
\toprule
Method 
& SSIM $\uparrow$ 
& PSNR $\uparrow$ 
& MSE $\downarrow$ 
& SSIM-L $\uparrow$ 
& LPIPS $\downarrow$ \\
\midrule
Target-Only-GT     
& 0.573 {\scriptsize $\pm$ 0.001}
& 16.47 {\scriptsize $\pm$ 0.10}
& 0.0226 {\scriptsize $\pm$ 0.0005}
& 0.543 {\scriptsize $\pm$ 0.004}
& 0.241 {\scriptsize $\pm$ 0.001} \\
Shared-GT          
& 0.731 {\scriptsize $\pm$ 0.002}
& 17.14 {\scriptsize $\pm$ 0.06}
& 0.0193 {\scriptsize $\pm$ 0.0003}
& 0.683 {\scriptsize $\pm$ 0.003}
& 0.136 {\scriptsize $\pm$ 0.002} \\
Target-Only-Latent 
& 0.581 {\scriptsize $\pm$ 0.001}
& 17.01 {\scriptsize $\pm$ 0.02}
& 0.0199 {\scriptsize $\pm$ 0.0001}
& 0.544 {\scriptsize $\pm$ 0.001}
& 0.227 {\scriptsize $\pm$ 0.001} \\
Shared-Latent      
& 0.756 {\scriptsize $\pm$ 0.002}
& 18.26 {\scriptsize $\pm$ 0.10}
& 0.0149 {\scriptsize $\pm$ 0.0004}
& 0.729 {\scriptsize $\pm$ 0.004}
& 0.110 {\scriptsize $\pm$ 0.002} \\
SCAR-kl            
& 0.763 {\scriptsize $\pm$ 0.001}
& 18.51 {\scriptsize $\pm$ 0.05}
& 0.0141 {\scriptsize $\pm$ 0.0002}
& 0.738 {\scriptsize $\pm$ 0.002}
& 0.105 {\scriptsize $\pm$ 0.001} \\
SCAR-grl           
& 0.761 {\scriptsize $\pm$ 0.002}
& 18.42 {\scriptsize $\pm$ 0.09}
& 0.0144 {\scriptsize $\pm$ 0.0003}
& 0.733 {\scriptsize $\pm$ 0.004}
& 0.109 {\scriptsize $\pm$ 0.003} \\
SCAR-kl-grl        
& \textbf{0.770} {\scriptsize $\pm$ 0.001}
& \textbf{18.70} {\scriptsize $\pm$ 0.05}
& \textbf{0.0135} {\scriptsize $\pm$ 0.0002}
& \textbf{0.747} {\scriptsize $\pm$ 0.004}
& \textbf{0.101} {\scriptsize $\pm$ 0.001} \\
\bottomrule
\end{tabular}
\end{table}

\section{Broader Impact, Limitations, and Future Work}
\label{app:broader_impact_limitations}

\paragraph{Potential positive impacts.}
SCAR studies self-supervised continuous action representation learning for embodied world models. Its main contribution is a transferable latent action interface that reduces reliance on embodiment-specific raw command representations. If developed responsibly, such representations may improve data efficiency, support transfer across robot platforms, and enable more simulation-based testing before real-world deployment. This may benefit applications such as robot manipulation, laboratory automation, assistive robotics, and operation in remote or hazardous environments.

\paragraph{Potential negative impacts.}
The main risk is over-interpreting simulation-based prediction results as evidence of safe real-world control. SCAR is evaluated on world-model prediction and transfer diagnostics, not closed-loop deployment. If the learned latent interface were used in downstream robotic systems without additional validation, model errors or simulation-to-real mismatches could lead to unreliable behavior. 

\paragraph{Limitations and future work.}
Our experiments are conducted on controlled simulation/control benchmarks and focus on world-model prediction rather than closed-loop planning or real-robot execution. This setting helps isolate cross-embodiment action representation learning, but it does not fully capture the diversity, noise, and safety constraints of real-world robot deployment. Future work should scale SCAR to larger and more diverse real-robot datasets, study closed-loop planning and policy learning with the learned latent action interface.

\paragraph{Mitigation and scope.}
We do not claim that SCAR is ready for safety-critical deployment. Any downstream use should require task-specific safety constraints, real-world validation, monitoring under distribution shift, and human oversight. Released code and models should be accompanied by documentation of intended use, limitations, dataset scope, and evaluation conditions.

\end{document}